\documentclass{article}
\usepackage{fullpage}

\usepackage{amsmath,amsfonts,amsthm,bm}

\def\eqref#1{equation~\ref{#1}}

\def\ceil#1{\lceil #1 \rceil}
\def\floor#1{\lfloor #1 \rfloor}
\def\1{\bm{1}}

\def\vone{{\bm{1}}}

\def\vb{{\bm{b}}}

\def\vh{{\bm{h}}}

\def\vp{{\bm{p}}}
\def\vq{{\bm{q}}}

\def\vs{{\bm{s}}}

\def\vv{{\bm{v}}}

\def\vx{{\bm{x}}}
\def\vy{{\bm{y}}}
\def\vz{{\bm{z}}}

\def\mI{{\bm{I}}}

\def\mM{{\bm{M}}}

\def\mS{{\bm{S}}}

\def\mV{{\bm{V}}}
\def\mW{{\bm{W}}}

\DeclareMathAlphabet{\mathsfit}{\encodingdefault}{\sfdefault}{m}{sl}
\SetMathAlphabet{\mathsfit}{bold}{\encodingdefault}{\sfdefault}{bx}{n}

\newcommand{\E}{\mathbb{E}}

\usepackage[pdftex]{graphicx} 
\usepackage{xcolor}
\usepackage{amsmath}
\usepackage{amssymb}
\usepackage{amsthm}
\usepackage{mathtools}
\usepackage{framed}
\usepackage{hyperref}
\usepackage{enumitem}
\usepackage[capitalize,noabbrev]{cleveref}
\usepackage{wrapfig}

\usepackage{dsfont}
\usepackage{caption,subcaption}
\usepackage{sidecap}

\usepackage{booktabs}
\usepackage{array}
\usepackage{multirow}

\usepackage{subcaption}

\newcommand{\Reach}[1]{\ensuremath{\operatorname{Rch}\left(#1\right)}}
\newcommand{\Ric}[1]{\ensuremath{\operatorname{Ric}\left(#1\right)}}
\newcommand{\poly}[1]{\ensuremath{\operatorname{poly}#1}}
\newcommand{\Vol}[1]{\ensuremath{\operatorname{Vol}#1}}
\newcommand{\sgn}[1]{\ensuremath{\operatorname{sign}#1}}
\newcommand{\sq}{\operatorname{SQ}}

\newcommand\restr[2]{{
  \left.\kern-\nulldelimiterspace 
  #1 
  \vphantom{\big|} 
  \right|_{#2} 
  }}

\theoremstyle{plain}
\newtheorem{theorem}{Theorem}[section]
\newtheorem{proposition}[theorem]{Proposition}
\newtheorem{lemma}[theorem]{Lemma}
\newtheorem{corollary}[theorem]{Corollary}
\theoremstyle{definition}
\newtheorem{definition}[theorem]{Definition}
\newtheorem{example}[theorem]{Example}
\newtheorem{assumption}[theorem]{Assumption}
\theoremstyle{remark}
\newtheorem{remark}[theorem]{Remark}

\title{Hardness of Learning Neural Networks under the Manifold Hypothesis}

\author{
{\sf Bobak T.\ Kiani}\thanks{John A. Paulson School of Engineering and Applied Sciences, Harvard University; e-mail: {\href{mailto:bkiani@g.harvard.edu}{\texttt{bkiani@g.harvard.edu}}}.}
\and
{\sf Jason Wang}\thanks{Harvard College, Harvard University; e-mail: {\href{mailto:jasonwang1@college.harvard.edu}{\texttt{jasonwang1@college.harvard.edu}}}.}
\and 
{\sf Melanie Weber}\thanks{John A. Paulson School of Engineering and Applied Sciences, Harvard University; e-mail: {\href{mailto:mweber@g.harvard.edu}{\texttt{mweber@g.harvard.edu}}}}
}

\date{ }

\begin{document}

\maketitle

\begin{abstract}
The \emph{manifold hypothesis} presumes that high-dimensional data lies on or near a low-dimensional manifold. 
While the utility of encoding geometric structure has been demonstrated empirically, rigorous analysis of its impact on the learnability of neural networks is largely missing.  
Several recent results have established hardness results for learning feedforward and equivariant neural networks under i.i.d. Gaussian or uniform Boolean data distributions. 
In this paper, we investigate the hardness of learning under the manifold hypothesis. We ask which minimal assumptions on the curvature and regularity of the manifold, if any, render the learning problem efficiently learnable. 
We prove that learning is hard under input manifolds of bounded curvature by extending proofs of hardness in the SQ and cryptographic settings for Boolean data inputs to the geometric setting.
On the other hand, we show that additional assumptions on the volume of the data manifold alleviate these fundamental limitations and guarantee learnability via a simple interpolation argument. 
Notable instances of this regime are manifolds which can be reliably reconstructed via manifold learning. 
Looking forward, we comment on and empirically explore intermediate regimes of manifolds, which have heterogeneous features commonly found in real world data.  
\end{abstract}

\section{Introduction}
High-dimensional data is often thought to have low-dimensional structure, which may stem, for instance, from symmetries in the underlying system. This observation has given rise to the \emph{manifold hypothesis}, which presumes that high-dimensional data lies on or near a low-dimensional manifold. Empirical studies have confirmed this hypothesis across domains~\cite{carlsson2009topology,pope2021intrinsic,dicarlo2012does,ma2011manifold}.
A plethora of algorithms for geometric data analysis~\cite{fukunaga1971algorithm,belkin2003laplacian,roweis2000nonlinear,cayton2008algorithms} and, more recently, machine learning~\cite{cohenc16,bronstein2017geometric,chen2023riemannian} seek to leverage such structure. While these methods have shown empirical promise, few formal results on the benefits of encoding geometric structure have been established. Here, we study the impact of the manifold hypothesis on the computational complexity of learning algorithms for neural networks. Specifically, we investigate, under which geometric assumptions feedforward neural networks are efficiently learnable. 

The computational hardness of learning shallow, fully-connected neural networks under i.i.d. Gaussian data distributions has been established by \cite{diakonikolas2020algorithms,chen2022hardness,goel2020superpolynomial,daniely2021local}, who show  
that the complexity of learning even single hidden layer neural networks grows at least exponentially with the input size. Some of this analysis utilizes the correlational statistical query (CSQ) framework, of which learnability via gradient decent is a notable instance. Recently, \cite{kiani2024hardness} have shown similar hardness results for equivariant neural networks, a class of geometric architectures that explicitly encode symmetries. Both lines of work indicate that additional assumptions on the neural network architecture or the data geometry are needed to establish learnability. Here, we focus on the latter. 

A separate body of literature investigates nonlinear, low-dimensional structure in data. Approaches for estimating intrinsic dimension~\cite{fukunaga1971algorithm,levina2004maximum,stanczuk2022your} and curvature~\cite{aamari2019estimating,sritharan2021computing,trillos2023continuum} from data seek to characterize the geometry of low-dimensional manifolds.
\emph{Manifold Learning} aims to identify low-dimensional structure by reconstructing low-dimensional manifolds from data. Methods such as Multi-Dimensional Scaling~\cite{kruskal1964multidimensional}, Isomap~\cite{tenenbaum2000global} or Diffusion Maps~\cite{diffusion-map} have shown empirical success in learning low-dimensional data manifolds. For some of these approaches, formal guarantees on their reliability with respect to the geometric characteristics and sample size of the data are known~\cite{bernstein2000graph}. More generally, the manifold learning problem and its complexity have been formally studied in~\cite{fefferman2016testing,fefferman2018fitting,aamari2022adversarial}. However, the impact of data geometry on the complexity of learning neural networks in downstream tasks remains open.

In this paper, we ask: \emph{Which assumptions on the data geometry guarantee learnability of neural networks?} 
We show that the manifold hypothesis on its own does not guarantee learnability. In particular, we give hardness results for a class of low-dimensional manifolds in the statistical query (SQ) model or under cryptographic hardness assumptions. 
Our work follows an established line of proof techniques in the SQ literature~\cite{daniely2021local,goel2020superpolynomial,chen2022hardness}, extending hardness results for neural network training in the Boolean and Gaussian input models to more general geometries.
We further show that additional assumptions on the volume and curvature of the data manifold alleviate the fundamental limitations and guarantee learnability through a rather simple interpolation argument. In particular, manifolds which can be reliably reconstructed via manifold learning are in this regime. We further discuss geometric regimes in which our results do not directly apply, and in which provable learnability remains an open question. We illustrate our learnability results through computational experiments on neural network training in the learnable and provably hard regimes. We further complement our theoretical analysis with a brief computational study of the intrinsic dimension of image data manifolds, with the goal of testing the geometric assumptions in our framework.

\section{Background}

\subsection{Basic Notation}
We denote scalars, vectors, and matrices as $v, \vv,$ and $\mV$ respectively. We consider ambient spaces $\mathbb{R}^n$ equipped with the usual inner product $\langle \cdot, \cdot \rangle$ and associated $\ell_2$ norm $\| \cdot \|$. Submanifolds $\mathcal{M} \subset \mathbb{R}^n$ considered in this study have intrinsic dimension $d\leq n$ and are compact and connected (unless otherwise stated). The tangent space $T_{\vp} \mathcal{M}$ at a point $\vp \in \mathcal{M}$ denotes the $d$-dimensional linear subspace of $\mathbb{R}^n$ spanned by velocity vectors of smooth curves incident at $\vp$. Given a subset $S \subset \mathbb{R}^n$, we denote by $d(\vz, S) = \inf_{\vp \in S} \| \vz - \vp \|$ the distance of a point $\vz$ to $S$. We denote $\Vol_d$ as the $d$-dimensional volume measure inherited from the $d$-dimensional Hausdorff measure and denote $\omega_d(r)$ and $\sigma_d(r)$ as the volume of the $d$-dimensional ball and $d$-dimensional sphere of radius $r$ respectively. For a point $\vp \in \mathcal{M}$ in a given manifold $\mathcal{M}$, $\Vol_{\mathcal{M}}(\vp, r) $ denotes the volume of the ball of radius $r$ around the point $\vp$ with respect to the Riemannian distance metric. 

\subsection{Learning Setting} We consider the task of learning feedforward neural networks $f:\mathbb{R}^n \to \mathbb{R}$ composed of $L$ hidden layers $f^{(\ell)}:\mathbb{R}^{d_{\ell}} \to \mathbb{R}^{d_{\ell-1}}$ taking the form
\begin{equation}
\begin{split}
    f(\vx) &= \vv^{\top} f^{(L)}\left( f^{(L-1)}\left( \cdots \left( f^{(1)}\left( \vx \right) \right)\cdots \right) \right), \\
    f^{(\ell)}(\vh) &= \operatorname{ReLU}\left( \mW_{\ell} \vh \right) + \vb_{\ell},
\end{split}
\end{equation}
where $\vv \in \mathbb{R}^{d_L}$, $\mW_{\ell} \in \mathbb{R}^{d_{\ell} \times d_{\ell-1}}$ and $\vb \in \mathbb{R}^{d_{\ell}}$ are trainable weights. The input dimension $d_0$ is set to the ambient dimension $n$ and the output is scalar-valued. Throughout we will consider the setting where the weight entries and hidden widths are bounded by $O(\poly(n))$. When guaranteeing that a class of networks is learnable, we will assume that the number of layers $L=O(1)$ is constant with respect to the input dimension. Our formal hardness results will apply for single hidden layer networks ($L=1$).

We consider learnability in the distribution-specific probably approximately correct (PAC) setting where the goal is to produce an algorithm which can learn a target function given polynomial time and samples \cite{mohri2018foundations}.
\begin{definition}[Efficiently PAC Learnable]
    A concept class $\mathcal{C}$ consisting of functions $c:\mathcal{X} \to \mathcal{Y}$ is \emph{efficiently PAC-learnable} over distribution $\mathcal{D}$ on $\mathcal{X} \subseteq \mathbb{R}^n$, if there exists an algorithm such that for any $\epsilon>0$ and $\delta>0$, and for any target concept $c^* \in \mathcal{C}$, the algorithm takes at most $m=\operatorname{poly}(1/\epsilon, 1/\delta, n)$ samples drawn i.i.d. from $(\vx, c^*(\vx))$ with $\vx \sim \mathcal{D}$, and returns a function $f$ satisfying $\|f-c^*\|_{\mathcal{D}} \coloneqq \sqrt{\mathbb{E}_{\vx \in \mathcal{D}}[(f(\vx) - c^*(\vx))^2]} \leq\epsilon$ with probability at least $1-\delta$ in time at most $\operatorname{poly}(1/\epsilon, 1/\delta, n)$.
\end{definition}
Note, that the above is a distribution specific instance of PAC learning as we require the algorithm to work only for a given distribution and not all distributions. In \Cref{sec:bounded_curvature}, we will also show a hard class of functions which is likely not efficiently PAC learnable. Hardness results are proven in the restricted \textbf{statistical query (SQ)} setting, a query complexity based model for proving hardness capturing most algorithms in practice \cite{kearns1998efficient,reyzin2020statistical}. 
Given a joint distribution $\mathcal{D}$ on input/output space $\mathcal{X} \times \mathcal{Y}$, any SQ algorithm is composed of a set of queries. Each query takes as input a function $g:\mathcal{X} \times \mathcal{Y} \to [-1,1]$ and tolerance parameter $\tau>0$, and returns a value $\sq(g,\tau)$ in the range:
\begin{equation}
    \left| \mathbb{E}_{(\vx, y) \sim \mathcal{D}}\left[ g(\vx,y) \right] -  \sq(g,\tau) \right| \leq \tau.
\end{equation}
Gradient based algorithms can be queried by, for example, setting $g(\vx,y) = C\frac{\partial}{\partial\theta} \left(\operatorname{NN}_\theta(\vx) - y \right)^2$ to estimate the gradient of a neural network $\operatorname{NN}_\theta$ with respect to parameter $\theta$ for the MSE loss (constant $C$ chosen so that outputs of $g$ are bounded in $[-1,1]$ forming a valid query). Hardness is quantified in the number of queries needed to learn a function $c^*$ drawn from function class $\mathcal{C}$. 

\paragraph{Manifold smoothness restrictions.}  To conform to practical settings where input features are normalized (e.g. image pixel values in the range $[0,1]$), input distributions are assumed to be supported on smooth $d$-dimensional sub-manifolds $\mathcal{M} \subset [0,1]^n$ of the $n$-dimensional hypercube. For any given manifold $\mathcal{M}$, we will assume that data distributions $\mathcal{D}_{\mathcal{M}}$ supported on that manifold have a smooth density $f$ with respect to the volume measure and are appropriately bounded such that $\rho_{\max} \coloneqq \frac{\max_{\vx \in \mathcal{M}} f(\vx)}{\min_{\vx \in \mathcal{M}} f(\vx) } = O(\poly(n))$.

We will also place curvature restrictions on the manifold by bounding its reach, a global smoothness quantity introduced by \cite{federer1959curvature} and commonly studied in the manifold learning community \cite{aamari2019estimating,aamari2022adversarial,fefferman2016testing,genovese2012minimax}. Informally, it is the largest number $D$, such that any point in ambient space at distance less than $D$ has a unique nearest neighbor in the manifold $\mathcal{M}$. It is defined more formally from descriptions of the medial axis (\Cref{fig:medial_axis}).

\begin{figure}[t]
    \centering
    \vspace{-1cm}
    \includegraphics[]{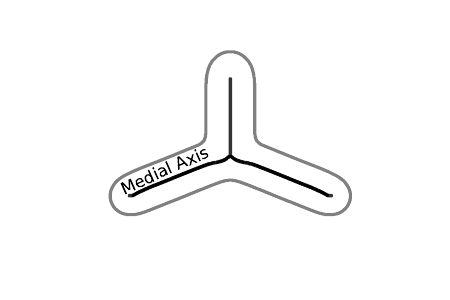}\vspace{-1.2cm} 
    \caption{Example of a one-dimensional manifold and its medial axis. Its reach is given by the minimum distance of the medial axis to the manifold.}
      \label{fig:medial_axis}
\end{figure}

\begin{definition}[Reach from medial axis]\label{def:reach}
    Given a closed subset $S \subset \mathbb{R}^n$, the medial axis $\operatorname{Med}(S)$ of $S$ consists of the set of points with no unique nearest neighbor:
    \begin{equation}
    \begin{split}
        \operatorname{Med}(S) = \{ &\vz \in \mathbb{R}^n: \exists \vp \neq \vq \in S, 
        \|\vp - \vz \| = \|\vq - \vz\| = d(\vz,S)   \}. 
    \end{split}
    \end{equation}
    The reach $\Reach{S}$ of $S$ is the minimal distance from $S$ to to $\operatorname{Med}(S)$:
    \begin{equation}
        \Reach{S} = \inf_{\vz \in \operatorname{Med}(S)} \operatorname{dist}(\vz, S). 
    \end{equation}
\end{definition}

Bounds on the reach imply corresponding bounds on the radius of curvature (related to the geodesic and normal curvature) at any point in the manifold since $\Reach{\mathcal{M}}^{-1} \geq \|\gamma_{\gamma(t), \gamma'(t)}''(t)\|$ for any unit-speed geodesic $\gamma:\mathbb{R} \to \mathcal{M}$ where $\|\gamma_{\gamma(t)}'(t)\| = 1$ for all $t\in \mathbb{R}$ \cite{aamari2022adversarial}. 
We will encounter a second classical curvature notion, \emph{Ricci curvature}, which is a local, intrinsic curvature notion that characterizes the volume growth of geodesic balls (see sec.~\ref{apx:aux} for a more formal definition). Positive lower bounds on the Ricci curvature imply that the manifold has a bounded diameter, a fact that we will use below.


\subsection{Related works}
Here, we briefly summarize the most relevant prior work to our study and reserve \Cref{app:related_works} for a more detailed discussion. Our study lies at the intersection of research in manifold learning complexity and neural network learnability. In manifold learning, various works have analyzed the complexity of learning tasks over input manifolds. In one line of work, estimation of smooth manifolds in Hausdorff loss has been shown to require sample complexity of $O(\epsilon^{-d/2}\log(1/\epsilon))$, which is independent of the ambient dimension $n$ \cite{boissonnat2010manifold,genovese2012minimax,kim2015tight}; later works also provided algorithms that run in time linear in $n$ \cite{aamari2018stability,divol2021minimax}.  
In a learning setting similar to our work, \cite{aamari2022adversarial} provide upper and lower bounds on the complexity of manifold reconstruction in SQ settings.
In a separate context, \cite{narayanan2009sample,fefferman2016testing} show that the sample complexity for determining whether a dataset is within a class of manifolds of specified intrinsic dimension, curvature, and volume bounds grows exponentially with the intrinsic dimension and polynomially in the reach and volume. \cite{narayanan2009sample} also categorizes the sample complexity for binary classification over smooth cuts on a data manifold where smoothness is defined by the condition number of the classification boundary of the manifold (a quantity closely related to reach). 

The hardness of learning neural networks has a long history that we detail further in \Cref{app:related_works}. Under the i.i.d. Gaussian input model, superpolynomial~\cite{goel2020superpolynomial} and exponential~\cite{diakonikolas2020algorithms} lower bounds for learning single hidden layer networks in CSQ settings have been shown. For uniformly random Boolean inputs, \cite{daniely2021local} reduce the problem of learning single hidden layer neural networks to a cryptographically hard problem, a technique, which we also use in \Cref{app:crypto_equivalent_hardness}. They also show how to ``Booleanize" Gaussian inputs to show hardness for three hidden layer networks, which was later extended to SQ and cryptographic hardness of learning two hidden layer networks in \cite{chen2022hardness}. 
\cite{kiani2024hardness} used similar techniques to give hardness results for equivariant neural networks. To the best of our knowledge, the hardness of learning feedforward neural networks under more general data geometries, such as under the manifold hypothesis, has not been studied previously.
Various works have found efficient learning algorithms under i.i.d. input assumptions when the weights of the networks are restricted in their rank, condition number, positivity, and other criteria \cite{vempala2019gradient,chen2023faster,bakshi2019learning,diakonikolas2020algorithms}. We consider the generic setting where weight matrices are bounded in width and magnitude by $O(\poly(n))$, but otherwise unrestricted.


\section{Learnability results}


Our main results prove the existence of a learnable and hard to learn class of input data manifolds illustrated in \Cref{fig:regimes}. The \emph{learnable} 
setting is the class of efficiently sampleable manifolds, which form the basis for algorithms that can provably reconstruct manifolds \cite{bernstein2000graph,cheng2005manifold,narayanan2010sample,fefferman2018fitting}. As expected, we find that a simple interpolation argument guarantees learnability of neural networks over these manifolds. Below and in \Cref{app:examples_sampleable}, we comment on instances of data geometries commonly assumed in machine learning and data science applications that place manifolds within this regime. 

The \emph{provably hard}
setting are manifolds without bounds on volume but with bounds on curvature quantified globally by the reach. When the bound on the reach grows no faster than $o(\sqrt{n})$ ($n$ denoting the input size), we construct input data manifolds that feature curvature no larger than the stated bounds but for which learning neural networks is exponentially hard. Our proofs construct curves that cover exponentially many quadrants of the Boolean cube, which allows us to extend classical hardness results for learning Boolean functions expressible by neural networks. 

\begin{figure}
  \begin{minipage}[c]{0.45\textwidth}
    \includegraphics[width=\textwidth]{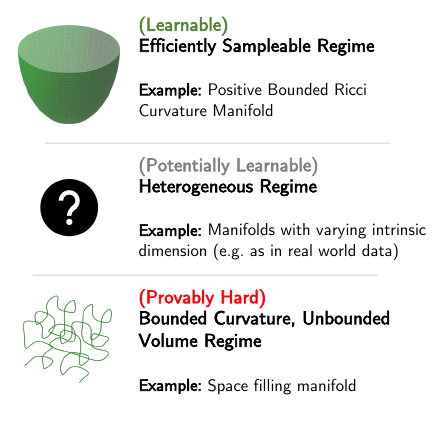}
  \end{minipage}\hfill
  \begin{minipage}[c]{0.53\textwidth}
    \caption{
       Learnability of neural networks depends on the regularity and smoothness properties of the input data manifold. In the efficiently sampleable regime corresponding to manifolds which can be approximated well with samples, neural networks are learnable via simple interpolation arguments. In the regime where manifolds are bounded solely by their curvature and intrinsic dimension, we show classes of manifolds that obstruct the learnability of algorithms. Real-world data likely lives in an intermediate regime with heterogeneous properties (e.g. manifolds with varying intrinsic dimension; see \Cref{sec:experimental_geometry}).
    } \label{fig:regimes}
  \end{minipage}
\end{figure}

Within the manifold regimes we study, our results are relatively tight with respect to bounds on the reach, since whenever the reach grows as $\omega(\sqrt{n})$, such manifolds fall within the first class of efficiently sampleable manifolds. 
We leave as an open question the learnability of manifolds whose reach grows exactly as $\Theta(\sqrt{n})$ (see \Cref{app:volume_bound_reach} for more details). Looking beyond our setting, real-world data manifolds feature heterogeneous properties that may not conform to the global bounds on curvature, intrinsic dimension, etc. that we set in our study. Better characterizing these heteregeneous features is a first step to extending our learnability results to more realistic settings. We conduct some preliminary empirical analysis in analyzing this heterogeneity in our experiments (\Cref{sec:experimental_geometry}).

\subsection{Sampleable regime} \label{sec:manifold_reconstruction}
In the manifold reconstruction literature, the goal is to draw samples from a distribution $\mathcal{P}$ (typically supported on a given manifold $\mathcal{M}$) and find a manifold $\mathcal{M'}$ which closely approximates the sample distribution or target manifold in some appropriate error metric such as the Haussdorf distance   \cite{tenenbaum2000global,bernstein2000graph,aamari2021statistical,fefferman2016testing,cheng2005manifold,narayanan2010sample,fefferman2018fitting}. Runtime and sample complexity for these algorithms are generally at least linear in the volume of the manifold, polynomial in smoothness parameters such as the reach, and exponential in the intrinsic dimension, though the dependence on the ambient dimension can often be removed \cite{fefferman2018fitting}. Manifold reconstruction algorithms offer a direct means for learning data from a target neural network $f^*$ as one can apply such algorithms to the graph of the function consisting of points $(\vx, f^*(\vx)) \in \mathcal{X} \times \mathbb{R}$. As we will show here, guarantees of learning in the manifold reconstruction literature directly imply guarantees for learning neural networks via a simple interpolation argument. However, caution must be taken in assuming this situation holds in practice. For this procedure to be efficient, the volume of the manifold should be at worst polynomial in the intrinsic dimension $d$ and not growing significantly with the ambient dimension $n$. Some empirical evidence suggests that manifolds of real-world data may not be in this regime \cite{peng2021hyperbolic,brown2022verifying}.

Essential to many of the manifold reconstruction algorithms is the requirement that one can efficiently cover the manifold with samples \cite{cheng2005manifold,fefferman2016testing,bernstein2000graph}. This requirement comes in various technical forms and we frame our results here assuming the ability to form an epsilon net with samples.

\begin{definition}[$(\epsilon,\delta)$-net]
     Given a distribution $\mathcal{D}_{\mathcal{M}}$ over a manifold $\mathcal{M}$, a subset $\mathcal{S} \subset \mathcal{M}$ of points forms an $(\epsilon,\delta)$-net if with probability at least $1-\delta$ over draws $\vx \sim \mathcal{D}_{\mathcal{M}}$, there exists a point $\vx' \in \mathcal{S}$ where $\|\vx - \vx'\| \leq \epsilon$. 
\end{definition}

We denote a sequence of manifolds indexed by ambient dimensions as efficiently sampleable if for a fixed intrinsic dimension $d$, one can draw polynomially many samples in $n$ and error $1/\epsilon$ to form an epsilon net over the manifold $\mathcal{M}$.

\begin{definition}[Efficiently sampleable manifold] \label{def:efficiently_sampleable}
    Let $\{\mathcal{M}_n\}$ denote a sequence of manifolds in ambient dimension $n$ with fixed intrinsic dimension $d=O(1)$. Let $\{\mathcal{D}_{\mathcal{M}_n}\}$ be a corresponding sequence of distributions over points on the manifolds. We denote this sequence of manifolds as \emph{efficiently sampleable} if with at most $\tilde{O}(\operatorname{poly}(n, 1/\epsilon))$ samples drawn i.i.d. from $\mathcal{D}_{\mathcal{M}_n}$, one can form an $(\epsilon,\delta)$-net of the manifold $\mathcal{M}_n$ for any $\delta = \Omega(\poly{(n)}^{-1})$.
\end{definition}

Note that classical manifold learning algorithms typically implicitly assume efficiently sampleable manifolds as they assume that the algorithms run efficiently or have runtimes that depend polynomially on the manifold's volume $\operatorname{Vol}(\mathcal{M})$ (see \Cref{ex:bernstein} for details).
With standard regularity assumptions on the manifold, such a volume assumption typically implies that an epsilon net can be formed efficiently with $\tilde{O}(\operatorname{Vol}(\mathcal{M})/\epsilon^d)$ samples by a coupon collector argument (see, e.g., \Cref{ex:isoperimetric_setting}). 

\begin{remark}
    Given a target function $f^*$ with inputs on a manifold $\mathcal{M} \subseteq \mathcal{X}$, we could treat the graph of the function $\mathcal{M}_{f^*} = \{ (\vx, f^*(\vx)): \vx \in \mathcal{M} \}$ as a manifold in $\mathcal{X} \times \mathbb{R}$. Then, one can apply manifold reconstruction algorithms directly to the graph $\mathcal{M}_{f^*}$, which lies in an ambient space of dimension $n+1$. This likely works in practice, though it runs into some technical issues due to discontinuities introduced by ReLU activations. We apply a more direct method here.
\end{remark}

\begin{proposition}
    Let $n$-dimensional inputs be drawn from a sequence of efficiently sampleable manifolds $\mathcal{M}_n$ with intrinsic dimension $d=O(1)$ and distributions $\mathcal{D}_{\mathcal{M}_n}$ over the manifold. Denote $\mathcal{H}_n$ as the function class of constant depth ReLU networks on $n$ inputs with weights bounded in magnitude by $B=O(\poly(n))$ and $O(\poly(n))$ width. Then, one can learn $f^* \in \mathcal{H}_n$ up to error $\epsilon$ in runtime and sample complexity $O(\operatorname{poly}(n,B, 1/\epsilon))$.
\end{proposition}
\begin{proof}
    From the efficiently sampleable property of the manifold, we form an $(\epsilon',\delta)$-net using $\tilde{O}(\operatorname{poly}(n, 1/\epsilon))$ many samples where for all but a $\delta$-fraction of the manifold any point is within $\epsilon'$ of a given point within the net. Any network of $L$ layers and $O(\poly(n))$ bounded width and weight magnitude has the Lipschitz property that \cite{virmaux2018lipschitz}
    \begin{align}
        \|f^*(\vx) - f^*(\vx')\| &\leq \|\vx - \vx'\|\prod_{\ell=1}^L \|W^{(\ell)}\|_2 \\
        &\leq \|\vx - \vx'\| \prod_{\ell=1}^L \|W^{(\ell)}\|_F \nonumber \\
        &\leq O(B'n^L) \|\vx - \vx'\|, \nonumber
    \end{align}
    where $B'=O(\poly(n,B))$ is a bound on the Frobenius norm of the weight matrices. Therefore, we can interpolate the value of $f^*$ for any $\vx \in \mathcal{M}_n$ from a sample by setting $\epsilon' = O(\epsilon B'^{-1}n^{-L})$. For the $\delta$-fraction that is not in this net, note that $|f(\vx)| \leq O(\poly(n))$ for all $\vx $ in the hypercube so setting $\delta = o( |f(\vx)|^{-1})$ will guarantee that the contribution from these points decays. Finally, since $L=O(1)$ and $B=O(\poly(n))$ by assumption, we have the resulting polynomial complexity.
\end{proof}

Restrictions on the curvature or convexity properties of a manifold can often render the manifold efficiently sampleable. 
To illustrate that many manifolds implicitly fall within the regime of efficiently sampleable manifolds, we give an example below. Additional examples can be found in~\Cref{app:examples_sampleable}.

\begin{proposition}[Bounded Ricci curvature (Isoperimetric setting, see \cite{gromov1986isoperimetric,bubeck2023universal} for motivation) ]
\label{ex:isoperimetric_setting}
    The set of distributions $\mathcal{P}_{\mathcal{M}}$ over complete manifolds with bounded Ricci curvature $\Ric{\mathcal{M}} \geq (d-1)K$ for a constant $K>0$\footnote{Here, the factor $(n-1)$ accounts for the standard scaling of the Ricci curvature for a sphere of radius $1/\sqrt{K}$ which has Ricci curvature $(n-1)K$.} are efficiently sampleable.
\end{proposition}
\begin{proof}
    The bound on the Ricci curvature guarantees that the diameter of the manifold is at most $2\pi/\sqrt{K}$ and the manifold is contained within a $d$-dimensional ball of radius $1/\sqrt{K}$ \cite{myers1941riemannian} (see also Theorem 6.3.3 of \cite{petersen2006riemannian}). Thus, we can form an $(\epsilon,\delta)$-net over the manifold $\mathcal{M}$ by inducing it from an $\epsilon$-cover on the ball of radius $1/\sqrt{K}$ (see \Cref{def:cover_and_packing} for definition). To achieve an $\epsilon$-cover of such a ball in $d$ dimensions, it suffices to have $ N_{\epsilon} = O(1/\epsilon^d)$ points (see, e.g., Lemma 5.2 of \cite{vershynin2010introduction}). We denote the balls in the cover as $b_1, \dots, b_{N_{\epsilon}}$. Let $\tau$ denote a sampling ratio where we will take $\tau^{-2} = O(\poly(n))$ samples to form the $(\epsilon, \delta)$-net. By a coupon-collector argument (see \Cref{lem:coupon_collector}), $\tau^{-2}$ samples suffices to guarantee with high probability that any ball $b_i$ with probability at least $\mathcal{P}_{\mathcal{M}}(b_i) \geq \tau$ has a sample within it. $\delta $ then is at most the probability that a randomly drawn sample falls within a ball $b_i$ with probability $\mathcal{P}_{\mathcal{M}}(b_i) < \tau$. Therefore,
    \begin{equation}
        \delta \leq \sum_{i: \mathcal{P}_{\mathcal{M}}(b_i) < \tau} \mathcal{P}_{\mathcal{M}}(b_i) \leq N_{\epsilon} \tau. 
    \end{equation}
    Thus, setting $\tau = \delta / N_{\epsilon}$ suffices to achieve $\delta = \Omega(\poly(n)^{-1})$.
    
\end{proof}

We remark that the efficiently sampleable setting does not necessarily require that the manifold be fully connected. In fact, it is straightforward to extend proofs of learnability such as in \Cref{ex:isoperimetric_setting} to settings where there are multiple such disconnected manifolds as long as the number of disconnected components does not grow superpolynomially in $n$.

\subsection{Hard regime with bounded curvature}
\label{sec:bounded_curvature}
The learnability of the networks in the manifold learning setting relied crucially on the fact that such settings implicitly place bounds on the volume of such manifolds.  Here, we lift that restriction and consider a setting where manifolds are smooth and bounded only in their curvature and intrinsic dimension. More formally, we consider studying inputs drawn from manifolds with bounded reach (see \Cref{def:reach}).

In this section, we construct a sequence of manifolds $\{\mathcal{M}_n\}$ with sufficiently bounded reach that are exponentially hard to learn in the SQ model. These manifolds resemble space filling curves which wrap around exponentially many quadrants of the Boolean cube (see \Cref{app:space_filling_curve}). Given a bound on the reach $\Reach{\mathcal{M}_n\}} = O(n^\alpha)$ for $\alpha<0.5$, the space filling curve wraps around $2^{n^{1-2\alpha}}=2^{n^{\Omega(1)}}$ quadrants allowing one to extend standard hardness results for learning Boolean functions to data supported on this manifold.

\begin{theorem} \label{thm:hard_to_learn}
    Let $\mathcal{D}_{\mathcal{M}_n}$ denote the uniform distribution over manifolds $\mathcal{M}_n$ constructed in \Cref{def:manifold_sequence} where $\Reach{\mathcal{M}_n} = O(n^\alpha)$ for $\alpha<0.5$. 
    Any SQ algorithm $\mathcal{A}$ capable of learning the class of linear width single hidden layer $\operatorname{ReLU}$ neural networks under this sequence of distributions up to mean squared error sufficiently small ($\epsilon/8$ suffices) with queries of tolerance $\tau$ must use at least $\Omega(\tau^2 2^{n^{\Omega(1)}})$ queries. 
\end{theorem}
\begin{proof}[Proof sketch.]
    We first form manifolds conforming to the stated bounds on the reach and intrinsic dimension while also ``looping" around exponentially many quadrants of the hypercube. This is done by forming space-filling curves whose paths follow the indexing of a Gray code~\cite{gray1953pulse}. Inputs on this manifold are real-valued, but by carefully selecting a portion of the input dimensions, we obtain inputs that with high probability are approximately distributed as uniform over Boolean inputs $\{0,1\}^{n_b}$. We then extend previous proofs of the hardness of learning networks approximating parity functions under i.i.d. Boolean inputs proven in \cite{daniely2021local,chen2022hardness} to our setting. See \Cref{app:proofs_for_bounded_curvature} for complete proofs.
\end{proof}

\begin{remark}
    In \Cref{thm:cryptographic}, we show an equivalent proof of hardness in the cryptographic setting following the techniques in \cite{daniely2021local}. Namely, in \Cref{app:crypto_equivalent_hardness}, we show that the class of functions in \Cref{thm:hard_to_learn} is also hard to learn conditional on cryptographic assumptions related to the hardness of learning a class of pseudorandom functions. 
\end{remark}

The above result is relatively tight with respect to the bound on the reach. Namely, whenever the reach $\Reach{\mathcal{M}_n} = \omega(n^\alpha)$ is growing faster than $\sqrt{n}$, one can show that the volume of such manifolds is at most $O(\poly(n))$ and fitting within the regime of efficiently sampleable manifolds studied in the previous subsection. 

\begin{proposition} 
    Given a sequence of manifolds $\{\mathcal{M}_n\}$ of intrinsic dimension $d$ (fixed and independent of $n$) and reach bounded by $\Reach{\mathcal{M}_n} = \omega(n^{0.5})$, the volume of the manifolds grows at most $\Vol_d(\mathcal{M}_n) = O(\poly(n))$.
\end{proposition}

We refer the reader to \Cref{app:volume_bound_reach} for further details and formal proofs of this tightness.

\section{Experiments}
\begin{figure*}[t!]
    \centering
    \vspace*{0mm}
    \begin{subfigure}[b]{0.48\textwidth}
        \centering
        \includegraphics[width=0.8\textwidth]{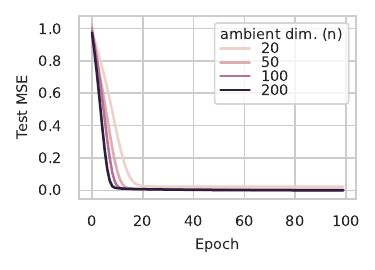}
        \vspace*{-3mm}
        \caption{Learnable Setting }
        \label{fig:isoperimetric}
    \end{subfigure}
    \hfill
    \begin{subfigure}[b]{0.48\textwidth}
        \centering
        \includegraphics[width=0.8\textwidth]{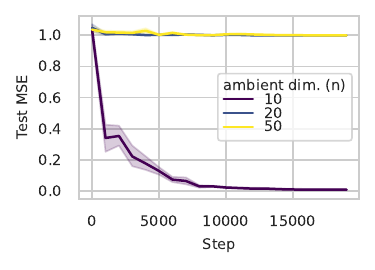}
        \vspace*{-3mm}
        \caption{Hard to Learn Setting}
        \label{fig:parity}
    \end{subfigure}
    \vspace*{0mm}
    \caption{(a) Learning is successful when inputs are drawn from a $d=10$ intrinsic dimensional hypersphere living in ambient space of dimension $n$ -- an instance of the bounded positive curvature model in \Cref{ex:isoperimetric_setting}. Target functions are single hidden layer networks taken from the class of hard to learn functions in the Gaussian i.i.d. input model \cite{diakonikolas2020algorithms}, which are no longer hard to learn in the input distribution considered here.  
    (b) When the ambient dimension is large, learning algorithm struggles to learn a single hidden layer neural network drawn from the class of functions in the setting of \Cref{thm:hard_to_learn} where the input data manifold has intrinsic dimension $d=1$ and reach $R=0.5$. The network trained to learn this target function is over-parameterized with respect to the target.
    Data is aggregated over five random realizations. }
    \vspace{-3mm}
    \label{fig:experiment_performance}    
\end{figure*}

\subsection{Empirical verification of main findings}
To verify that neural networks are learnable in the sampleable regime (\Cref{sec:manifold_reconstruction}) and hard to learn in the bounded curvature regime (\Cref{sec:bounded_curvature}), we train neural networks over input data manifolds taken from these regimes to confirm the formal theoretical results. We consider target networks, which are single hidden layer neural networks of $O(n)$ width and train overparameterized neural networks of larger width. For the sampleable regime, we draw inputs from a $d$-dimensional hypersphere supported over $d$ orthogonal dimensions in the $n$-dimensional ambient space. This is an instance of a complete manifold of bounded curvature as in \Cref{ex:isoperimetric_setting}. For the hard to learn regime, we draw inputs from the 1-dimensional manifold constructed in \Cref{app:space_filling_curve} with reach $R=0.5$. Target functions in this hard regime correspond to the parity functions described in our proofs in \Cref{app:proofs_for_bounded_curvature}.

The results in \Cref{fig:experiment_performance} empirically confirm our main findings. When inputs are drawn from the $d$-dimensional hypersphere (\Cref{fig:isoperimetric}), the trained neural network achieves low test error with a fixed training set of size $1000$ for all $n$. Here, target functions are drawn from those in \cite{diakonikolas2020algorithms} who provided a class of hard to learn functions in the i.i.d. Gaussian input model. Once the input model is changed to the $d$-dimensional hypersphere, this class of functions is no longer hard to learn (see \Cref{app:experiments} for similar results over random targets). 

In contrast, when inputs are drawn from the hard to learn manifold (\Cref{fig:parity}), learning is only possible when the ambient dimension is small. Here, target functions are randomly chosen parity functions over the inputs (see \Cref{app:experiments}). At each step of training, we provide the algorithm with a fresh batch of i.i.d. samples. In \Cref{fig:parity_3layers}, we replicate these results when training networks overparameterized in depth (i.e. three hidden layers as opposed to one) as well. We refer the reader to \Cref{app:experiments} for further details.

\subsection{Empirical study of geometry of data manifolds} \label{sec:experimental_geometry}
We empirically investigate the intrinsic dimension of real-world data manifolds as a first step towards testing the manifold regularity assumptions of our framework on real data. Largely, our results corroborate findings in other works highlighting the heterogeneous nature of real-world data manifolds \cite{pope2021intrinsic,stanczuk2022your,brown2022verifying}.

\paragraph{Experimental Setup.}
We estimate intrinsic dimension using samples generated by a diffusion model, closely following the approach of~\cite{stanczuk2022your}.
The use of a diffusion model allows us to generate arbitrarily dense samples in a neighborhood of a given point from the data manifold. Given these samples, we perform PCA on the collection of score vectors $\{\vs_i\}$ of the diffusion model at each of these samples, which point towards the direction of de-noising and hence towards the manifold itself. Estimating the rank of such a collection of score vectors recovers an estimate of the normal and intrinsic dimensions. We defer a detailed description of our approach, its implementation, and hyperparameter choices to \Cref{app:in-dim-exp}.

\begin{table*}[ht]
\centering
\begin{tabular}{lcccccccccc}
\toprule
\textbf{Data set} & \textbf{\#0} & \textbf{\#1} & \textbf{\#2} & \textbf{\#3} & \textbf{\#4} & \textbf{\#5} & \textbf{\#6} & \textbf{\#7} & \textbf{\#8} & \textbf{\#9} \\
\midrule 
MNIST (28 x 28) & 102 & 66 & 120 & 109 & 73 & 101 & 109 & 72 & 114 & 104 \\
KMNIST (28 x 28) & 180 & 199 & 134 & 169 & 135 & 128 & 149 & 191 & 205 & 199\\
FMNIST (28 x 28) & 429 & 177 & 596 & 275 & 225 & 233 & 312 & 125 & 201 & 418 \\
\bottomrule
\end{tabular}
\caption{Estimated intrinsic dimension shown for each of the ten classes in MNIST, KMNIST, FMNIST.}
\label{tab:dim-class}
\end{table*}
\begin{figure*}[ht]
    \centering
    \includegraphics[width=0.33\linewidth]{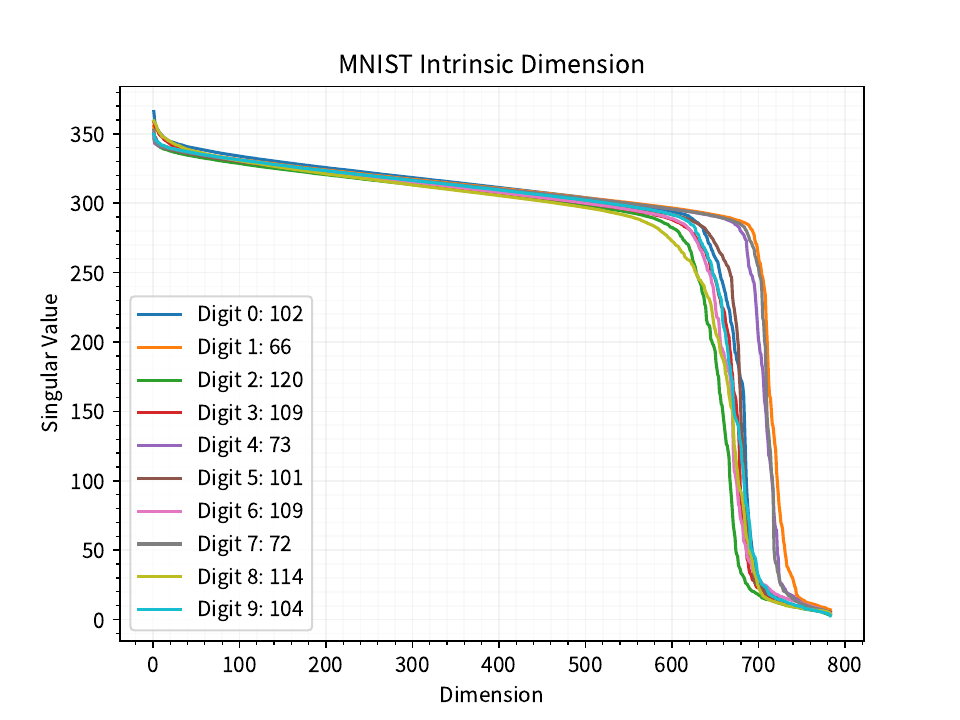}~
    \includegraphics[width=0.33\linewidth]{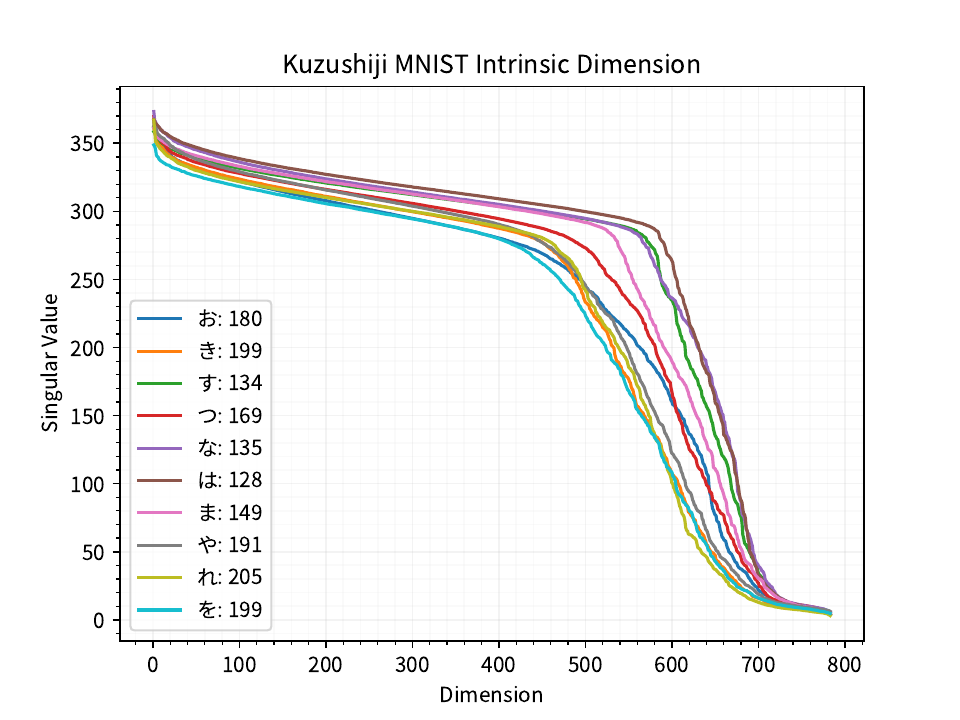}~
    \includegraphics[width=0.33\linewidth]{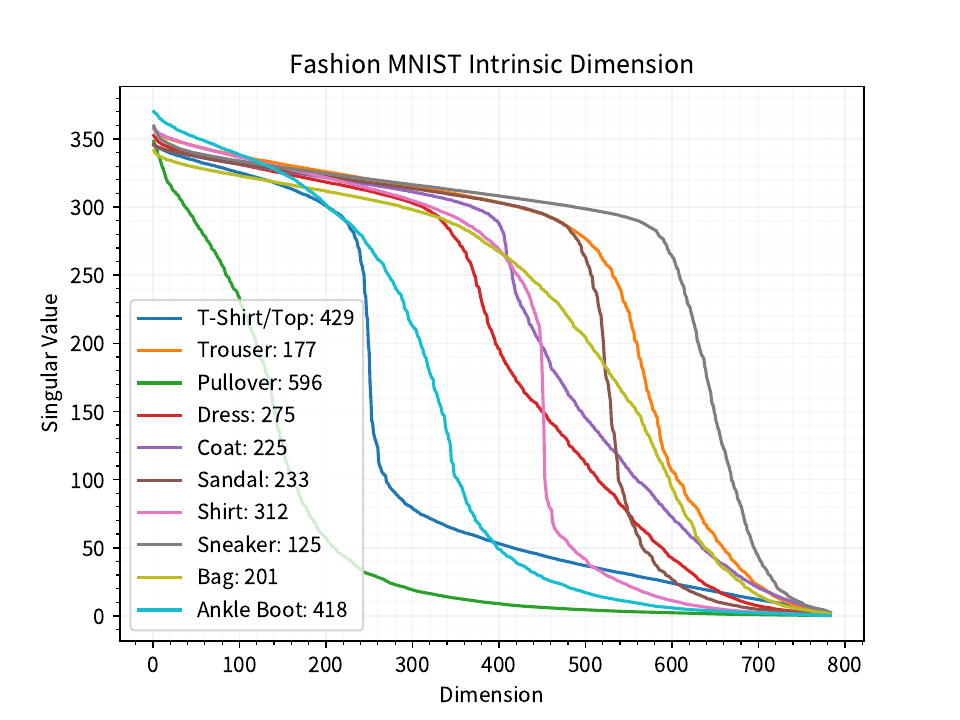}
    \caption{Sample spectrum of the singular values of a collection of score vectors given by a trained diffusion model pointing towards the direction of de-noising an image.  } 
    \label{fig:fmnist-score}
\end{figure*}

\paragraph{Results.} Estimated intrinsic dimensions for three image data manifolds are given in \Cref{tab:dim}. Score spectra for the three manifolds are shown in \Cref{fig:fmnist-score} confirming that the score vectors do indeed cluster in a linear subspace spanning the normal dimensions. 

Furthermore, our approach gives a similar estimate for the intrinsic dimension of MNIST as reported in \cite{stanczuk2022your}. To the best of our knowledge, FMNIST and KMNIST were not investigated in prior work. Comparing estimated intrinsic dimension across classes indicates significant heterogeneity of the real data manifolds, which would place them in the \emph{potentially learnable} regime in our proposed framework. This may seem surprising as images drawn from different classes can have different properties and symmetries. However, it corroborates our initial hypothesis that real-world data may not conform to global bounds on intrinsic dimension. Furthermore, we find that upon re-sizing the images to a smaller scale, the ambient dimension decreases faster in comparison to the intrinsic dimension for some datasets (see \Cref{tab:dim} in Appendix) lending some limited evidence to the underlying assumption in the manifold hypothesis that bounds on intrinsic dimension do not grow in tandem with ambient dimension. Additional experimental results and further discussion can be found in \Cref{app:real-mf}.

\section{Discussion}

In this paper we have investigated whether regularity assumptions on the data geometry can render feedforward neural networks efficiently learnable. We show that bounding curvature in low-dimensional data manifolds, a common assumption in high-dimensional learning, does not suffice to alleviate the fundamental hardness of learning such architectures. However, we establish learning guarantees under geometric assumptions common in the manifold reconstruction literature that allow manifolds to be efficiently covered with $O(\poly(n))$ samples. Our results contribute to a recent body of literature that seeks to establish learnability guarantees and provable hardness results for neural network architectures~\cite{blum1988training,chen2022hardness,diakonikolas2020algorithms,kiani2024hardness,goel2020superpolynomial,vempala2019gradient}. 

While our results establish guarantees for a wide range of data manifolds, including manifolds that can be provably reconstructed via manifold learning, there are geometric regimes in which the question of learnability remains open. We hypothesize that the heterogeneity of real data manifolds might place them in this more uncertain regime and find some evidence of this heterogeneity in our experiments. A further theoretical and empirical investigation of this class of data manifolds is an important direction for future work. 

The experimental results on estimating the intrinsic dimension of data manifolds are merely a starting point for understanding realistic assumption on data geometry in machine learning architectures. In future work we hope to look at a wider range of synthetic and real data sets, as well as a wider range of geometric characteristics and estimation techniques. 

Broadly, our results indicate that assumptions beyond global smoothness of data manifolds are needed to guarantee learnability of neural networks in real-world settings. To circumvent hardness results, the particular form of the networks and their weights can be changed to for example incorporate symmetries via equivariant architectures or place restrictions on weights such as bounds on condition number or rank \cite{chen2023faster,bakshi2019learning}. 
Furthermore, some datasets live in a discrete input space (e.g., language data) which does not conform directly to Riemannian analysis. In these settings, adaptations of the manifold hypothesis to analysis via graphs and other combinatorial objects presents an interesting direction for future work \cite{rubin2020manifold,bronstein2017geometric}.

\section*{Acknowledgement}
The authors thank Andrew Chang and Adityanarayanan Radhakrishnan for insightful discussions. BTK and MW were supported by the Harvard Data Science Initiative Competitive Research Fund and NSF award 2112085. JW acknowledges support from the Harvard College Research Program (HCRP).

\bibliographystyle{abbrv}
\bibliography{main}

\clearpage
\onecolumn

\appendix

\section{Extended Related Works}
\label{app:related_works}

\paragraph{Hardness of learning neural networks.}
The first results showing the hardness of learning neural networks are those of \cite{judd1987learning,blum1988training} proving that proper learning of neural networks is an $\mathsf{NP}$ complete task. As mentioned earlier, a number of works prove hardness results for learning feedforward ReLU networks under the Gaussian i.i.d. data model \cite{goel2020superpolynomial,diakonikolas2020algorithms,song2017complexity,chen2022hardness,daniely2021local}. Many of these works prove hardness results in the statistical query model \cite{reyzin2020statistical,kearns1998efficient}.  Similar hardness results, also in the SQ model, have also been shown for classes of symmetric or equivariant neural networks \cite{kiani2024hardness}. These results show that additional assumptions on the network class are needed to prove learnability results. Such assumptions which lead to provable learnability include those on the condition number or positivity of weights \cite{bakshi2019learning,diakonikolas2020algorithms}, polynomial approximation guarantees  \cite{vempala2019gradient}, or bounds on width \cite{chen2023faster,chen2023learning}. When the target is not necessarily a feedforward network, \cite{abbe2023sgd} show that single hidden layer ReLU networks can efficiently learn functions with low so-called leap complexity categorized by the growth in the degrees of polynomials over a linear subspace for i.i.d. Gaussian or uniform Boolean inputs.

\paragraph{Intrinsic dimension estimation.}
There are many algorithms, dating back decades, which estimate the intrinsic dimension of a given set of data points including approaches based on PCA \cite{fukunaga1971algorithm,minka2000automatic}, nearest neighbor methods \cite{levina2004maximum}, and others \cite{camastra2002estimating,kegl2002intrinsic}.
In the context of modern large image datasets such as Imagenet, \cite{pope2021intrinsic} perform estimates of the intrinsic dimension which are noticeably smaller than the input dimension. 
Various works study scaling laws arguing that the intrinsic dimension may be a useful indicator of the complexity of a dataset correlating to the rate at which neural network performance scale with increased data or training \cite{sharma2020neural,bahri2021explaining}. 
These works estimate the intrinsic dimension via a maximum likelihood estimator calculating the intrinsic dimension in relation to the distances of a point to its nearest neighbors \cite{levina2004maximum,facco2017estimating}. In later work which estimates the intrinsic dimension using diffusion models, \cite{stanczuk2022your} empirically show that this MLE estimator can underestimate the intrinsic dimension for various toy datasets where the intrinsic dimension is known. For this reason, our experiments follow the approach in~\cite{stanczuk2022your}. 

\paragraph{Learning algorithms over data manifolds.}
Among the earliest algorithms for learning over data manifolds are those for dimensionality reduction and representation learning \cite{fukunaga1971algorithm,belkin2003laplacian,kambhatla1993fast,dasgupta2008random,roweis2000nonlinear}. As mentioned earlier, a number of works provide algorithms and statistical analysis for manifold reconstruction or learning where the goal is to construct a manifold that closely fits a target \cite{cayton2008algorithms,aamari2022adversarial,fefferman2018fitting}. Perhaps closest to our work are studies of the hardness of learning functions defined over input manifolds. \cite{narayanan2009sample} categorize the sample complexity for binary classification over smooth cuts on a data manifold where smoothness is defined by the condition number of the classification boundary of the manifold (a quantity closely related to the reach). When no restrictions are placed on the probability distribution of data on the manifold, they show that the VC dimension is unbounded when this condition number exceeds the reach of the manifold. For nicer distributions where there is an upper bound on the density of the distribution on the manifold, a covering number argument shows that the sample complexity is dependent on the intrinsic dimension, reach, and condition number but independent of the ambient dimension. \cite{liu2021besov} show that for a certain class of convolutional networks trained on data lying in a low-dimensional manifold, the sample complexity depends weakly on the ambient dimension. \cite{tahmasebi2023exact} study the sample complexity of learning a target dataset over inputs taken from manifolds in kernel settings where the target function is invariant over some group operation acting on the input data manifold. The goal of their work was to quantify the gains in sample complexity from enforcing invariance and their sample complexity bounds depend on the volume of the manifold and dimension of the quotient space.

In recent years, various works have empirically tested and analyzed various aspects and implications of the manifold hypothesis. \cite{brown2022verifying} argue that manifolds of data may have many connected components of potentially varying intrinsic dimension and provide some empirical estimates of intrinsic dimension by image class supporting this argument. Our empirical analysis of the intrinsic dimension also supports this point.
\cite{buchanan2020deep} consider the task of learning a binary classification task where the two classes are two separate disconnected one dimensional manifolds (curves) on the sphere. 
Under regularity conditions on the extrinsic curvature and Riemannian distance properties of the curves, they prove that sufficiently wide and deep feedforward networks can learn to separate the manifolds. Their proofs are based on a neural tangent kernel approach.
\cite{maloney2022solvable} theoretically and empirically study kernel models trained under settings where the latent dimension of the data can be varied.
They show that the rate of decay in the spectrum of the kernel is more strongly related to the scaling of the loss as opposed to the intrinsic dimension of the data with some discussion of potential connections between the two measures.
\cite{whiteley2023statistical} propose a statistical model of real-world data where low dimensional features and manifold structures can naturally emerge. 
\cite{narayanan2009sample} categorize the sample complexity for binary classification over smooth cuts on a data manifold where smoothness is defined by the condition number of the classification boundary of the manifold (a quantity closely related to reach). When no restrictions are placed on the probability distribution of data on the manifold, they show that the VC dimension is unbounded when this condition number exceeds the reach of the manifold. For nicer distributions where there is an upper bound on the density of the distribution on the manifold, a covering number argument shows that the sample complexity is dependent on the intrinsic dimension, reach, and condition number but independent of the ambient dimension. 

\paragraph{Generative modeling.}
From the perspective of generative modeling, various works have recently considered the task of sampling from low-dimensional manifolds embedded in high dimensional space. \cite{pidstrigach2022score} consider distributions supported over a data manifold of low intrinsic dimension. They provide sufficient conditions for score-based generative models to identify the support of the learned distribution matching that of the target manifold. \cite{de2022convergence} prove convergence results in Wasserstein distance between score-based models and target distributions which are supported on a manifold under the assumption that the score estimator is accurate in the $L^\infty$ norm. These results were subsequently improved in \cite{chen2023sampling} which only required a score estimator is accurate in the $L^2$ norm, among other improvements. For data supported on a low-dimensional linear subspace, \cite{chen2023score} provide an end-to-end sampling guarantee avoiding the curse of dimensionality and showing that the score function can be efficiently estimated. The resulting sampling algorithm in their work thus approaches the true distribution in total variation distance and recovers the linear subspace appropriately. 

\paragraph{Geometric Machine Learning.}
We briefly mention a related body of work, which studies machine learning algorithms, which directly leverage geometric structure in data. Such methods have shown empirical promise in a variety of domains~\cite{bronstein2017geometric}. Examples of such architectures are translation- and rotation-equivariant neural networks~\cite{cohenc16}, graph neural networks~\cite{kipf2016semi} and DeepSets for permutation-invariant inputs~\cite{zaheer2017deep}, as well as hyperbolic machine learning algorithms~\cite{ganea2018hyperbolic,weber2020robust}, which assume that data lies on a manifold of constant negative curvature.

\section{Deferred Proofs}
\subsection{Auxiliary statements}\label{apx:aux}

For completeness, we recall several classical notions from metric geometry.
\begin{definition}[Packing and covering numbers] \label{def:cover_and_packing}
    Given a set $S \subseteq \mathbb{R}^n$, the covering number $\operatorname{cv}_S(\epsilon)$ is the minimum number of balls of size $\epsilon$ needed to cover every point in $S$:
    \begin{equation}
        \operatorname{cv}_S(\epsilon) \coloneqq \min \left\{ k>0 \;|\; \exists \; \vx_1, \dots, \vx_k \in S \text{ such that } \forall \vx \in S, \exists \; i \in [k] \text{ such that } \|\vx_i - \vx\| \leq \epsilon  \right\}.
    \end{equation}
    The packing number $\operatorname{pk}_S(\epsilon)$ is the maximum number of disjoint $\epsilon$-balls that can be contained in $S$:
    \begin{equation}
        \operatorname{pk}_S(\epsilon) \coloneqq \max \left\{ k>0 \;|\; \exists \; \vx_1, \dots, \vx_k \in S \text{ such that } \forall i \neq j: \|\vx_i - \vx_j\| > 2\epsilon  \right\}.
    \end{equation}
\end{definition}

\begin{lemma}[Packing and covering duality] \label{lem:cover_pack_duality}
    For any compact set $S \subset \mathbb{R}^n$ and $\epsilon > 0$:
    \begin{equation}
        \operatorname{pk}_S(2\epsilon) \leq \operatorname{cv}_S(2\epsilon) \leq \operatorname{pk}_S(\epsilon).
    \end{equation}
\end{lemma}
\begin{proof}
    For the first inequality, assume by contradiction that $\operatorname{cv}_S(2\epsilon)<\operatorname{pk}_S(2\epsilon)$. Then, there must exist two points $\vx_i, \vx_j$ in the maximal packing, which are within the same $2\epsilon$-ball in the cover and thus have distance $\| \vx_i - \vx_j \| \leq 2\epsilon$ contradicting the assumption.  
    
    For the second inequality, note that any maximal packing with balls of radius $\epsilon$ is a $2\epsilon$-covering, because otherwise, there would exist a point $\vx \in S$ which is at least $2\epsilon$ distance away from all the packing points contradicting the definition of a packing.
\end{proof}

Before stating the next result, we give a brief definition of \emph{Ricci curvature}, which locally characterizes the curvature of $\mathcal{M}$ in a neighborhood of a point $x \in \mathcal{M}$. 
\begin{definition}[Ricci Curvature]
    Let $v \in T_x \mathcal{M}$ denote a unit vector and $\lbrace u_1, \dots, u_{m-1}, v \rbrace$ an orthonormal basis of $T_x\mathcal{M}$; $g_x$ denotes the inner product on $T_x\mathcal{M}$. Then the \textit{Ricci curvature} at $x$ along the direction $v$ is defined as
\begin{equation}\label{eq:ric}
    {\rm Ric}_x(v) := \frac{1}{m-1} \sum_{i=1}^{m-1} g_x({R(v,u_i)v},{u_i}) \; ,
\end{equation}
where $R(u,v)w := \nabla_u \nabla_v w - \nabla_v \nabla_u w - \nabla_{[u,v]}w$ is the Riemann curvature tensor, with $[u,v]$ denoting the Lie Bracket between $u$ and $v$.
\end{definition}

\begin{theorem}[Corollary of Bishop-Gromov Theorem \cite{bishop1964relation} (see \cite{petersen2006riemannian} Lemma 7.1.3)] \label{thm:bishop_gromov}
    Let $\Vol_{\mathcal{M}}(\vp, r)$ denote the volume of the ball of radius $r$ around the point $\vp$ with respect to the Riemannian distance metric. If a manifold $\mathcal{M}$ of intrinsic dimension $d$ has bounded Ricci curvature $\Ric{\mathcal{M}} \geq (d-1)K$, then for all $r>0$ and $\vp \in \mathcal{M}$, $\Vol_{\mathcal{M}}(\vp, r) \leq \Vol_{S_d^K}(r)$ where $\Vol_{S_d^K}(r)$ denotes the volume of a ball of radius $r$ around an arbitrary point $\vp \in S_d^K $ in the space $S_d^K$ with constant sectional curvature $K$ ($d$-sphere if $K>0$ or $d$-dimensional Hyperbolic space if $K<0$).
\end{theorem}

\begin{lemma}[Coupon collector bound \cite{erdHos1961classical}] \label{lem:coupon_collector}
    Let $\xi_1, \dots, \xi_N$ be i.i.d. random variables taking values in $[n]$ with uniform probability: $\mathbb{P}[\xi_i = k] = \frac{1}{n}$ for all $k \in [n]$. Let $T_n$ be the minimum value $T$ such that for all $k \in [n]$, there exists at least one $i \in [T]$ where $\xi_i = k$ (i.e. every value in $[n]$ is covered). Then, $\mathbb{E}[T_n] = \Theta(n \log (n))$ and 
    \begin{equation}
        \lim_{n \to \infty}\mathbb{P}\left[ T_n < n \log(n) + cn \right] = \exp( -\exp(-c)). 
    \end{equation}
\end{lemma}

\subsection{SQ lower bounds}
To show lower bounds in the statistical query model, we follow the technique employed in \cite{chen2022hardness,bogdanov2017pseudorandom}, which constructs hardness from statistical independence properties of a function class. Since our setting differs slightly from those in \cite{chen2022hardness,bogdanov2017pseudorandom}, we provide an adapted version of their constructions below.

\begin{definition}[$(1-\eta)$-pairwise independent; see also Definition C.1 of \cite{chen2022hardness}]
    Let $\mathcal{C}$ be a function class consisting of functions $f: \mathcal{X} \rightarrow \mathcal{Y}$. Let $\mathcal{D}$ be a distribution on $\mathcal{X}$. $\mathcal{C}$ is $(1-\eta)$-pairwise independent if there exists a finite subset $\widehat{\mathcal{Y}} \subseteq \mathcal{Y}$ such that with probability $(1-\eta)$ over $\vx, \vx' \in \mathcal{X}$ drawn independently from $\mathcal{D}$, the distribution of $(f (\vx), f (\vx'))$ for $f \sim \operatorname{Unif}(\mathcal{C})$ drawn uniformly at random from $\mathcal{C}$ is the product distribution $\operatorname{Unif}(\widehat{\mathcal{Y}}) \otimes \operatorname{Unif}(\widehat{\mathcal{Y}})$.
\end{definition}

The above agrees with Definition C.1 of \cite{chen2022hardness} apart from the extension that here we define the independence with respect to a finite subset $\widehat{\mathcal{Y}} \subseteq \mathcal{Y}$ of the output space.

\begin{theorem}[Theorem C.4 of \cite{chen2022hardness}]\label{thm:pairwise-ind-sq-lower}
	Let the function class $\mathcal{C}$  of functions $f: \mathcal{X} \rightarrow \mathcal{Y}$ be a $(1-\eta)$-pairwise independent function family with respect to\ a distribution $\mathcal{D}$ on $\mathcal{X}$. For any $f \in \mathcal{C}$, let $\mathcal{D}_f$ denote the distribution of $(\vx, f(\vx))$ where $\vx \sim \mathcal{D}$. Let $\mathcal{D}_{\operatorname{Unif}(\mathcal{C})}$ denote the distribution of $(\vx, y)$ where $\vx \sim \mathcal{D}$ and $y = f(\vx)$ for $f \sim \operatorname{Unif}(\mathcal{C})$. Any SQ learner able to distinguish the labeled distribution $\mathcal{D}_{f^*}$ for an unknown $f^* \in \mathcal{C}$ from the randomly labeled distribution $\mathcal{D}_{\operatorname{Unif}(\mathcal{C})}$ using bounded queries of tolerance $\tau$ requires at least $\frac{ \tau^2}{2\eta}$ such queries.
\end{theorem}
\begin{proof}
    We follow the proof in \cite{chen2022hardness}. Let $\phi : \mathcal{X} \times \mathcal{Y} \to [-1, 1]$ be any query made by the learner and set $\phi[f]=\E_{\vx \sim \mathcal{D}} [\phi(\vx, f(\vx))]$. Then,
    \begin{equation}
		\begin{split}
			\operatorname{Var}_{f \sim \operatorname{Unif}(\mathcal{C})} \left[\phi[f]\right] &= \E_{f} \left[ \phi[f] \phi[f] \right] - \E_{f} \left[\phi[f]\right] \E_{f'} \left[\phi[f']\right] \\
			&= \E_{f, f'} \left[ \E_{\vx} [\phi(\vx, f(\vx))] \E_{\vx'} [\phi(\vx', f(\vx'))] - \E_{\vx} [\phi(\vx, f(\vx))] \E_{\vx'} [\phi(\vx', f'(\vx'))] \right] \\
			&= \E_{\vx, \vx'} \E_{f, f'} \left[ \phi(\vx, f(\vx)) \phi(\vx', f(\vx')) - \phi(\vx, f(\vx)) \phi(\vx', f'(\vx')) \right]\\
            &\leq 2\eta .
		\end{split} 
    \end{equation}
  In the last line, we use the fact that for any $(1-\eta)$-pairwise independent class $\mathcal{C}$, the inner expectation is zero with probability at least $1- \eta$ over the choice of $\vx, \vx' \sim \mathcal{D}$ and at most 2 otherwise. 

  Since any SQ algorithm must work with any given value of the noise within the stated tolerance, consider the adversarial strategy where the SQ oracle responds to a query with $\overline{\phi} = \E_{f \sim \operatorname{Unif}(\mathcal{C})} [\phi[f]]$ whenever possible. By Chebyshev's inequality,
  \begin{equation}
      \mathbb{P}_{f \sim \operatorname{Unif}(\mathcal{C})} \left[ \big|\phi[f] - \overline{\phi}\big| > \tau \right] \leq \frac{\operatorname{Var}_{f \sim \operatorname{Unif}(\mathcal{C})}\big[\phi[f]\big]}{\tau^2} \leq \frac{2\eta}{\tau^2}. 
  \end{equation}
  So each such query only allows the learner to rule out at most a $\frac{2\eta}{\tau^2}$ fraction of $\mathcal{C}$. Thus to distinguish $\mathcal{D}_{f^*}$ from $\mathcal{D}_{\operatorname{Unif}(\mathcal{C})}$, the learner requires at least $\frac{ \tau^2}{2\eta}$ queries.
\end{proof}

\subsection{Additional examples of learnable manifolds}
\label{app:examples_sampleable}

Here, we provide additional examples of manifolds and their properties that render such manifolds learnable as detailed in \Cref{sec:manifold_reconstruction}. There, learnability of the distribution of the manifold corresponded to the property that one can efficiently approximate the manifold with an epsilon net. 

As mentioned in the main text, many algorithms in the manifold reconstruction literature implicitly assume that the manifold is efficiently sampleable. We give one such example below.

\begin{example}[Manifold Learning Setting~\cite{bernstein2000graph}]
\label{ex:bernstein}
    \cite{bernstein2000graph} provide provable guarantees for the Isomap algorithm in the sense of recovering geodesic distances on a manifold $\mathcal{M}$ via shortest-path distance on a similarity graph $G$ under the following assumptions (listed in Main Theorem A in~\cite{bernstein2000graph}), which informally require among other things: (1) the algorithm is given a set of samples $\{\vx_i\}$, which form a $\delta$-net over the manifold $\mathcal{M}$; (2) $\mathcal{M}$ is geodesically convex; (3) the reach (equivalently encoded as bounds on the radius of curvature and branch separation) is bounded as $\Reach{\mathcal{M}} = O(\delta)$. Assumption (1) suffices to guarantee learnability in our setting.
\end{example}

We also extend \Cref{ex:isoperimetric_setting} to incorporate manifolds with negative curvature. Negatively curved manifolds have volume that can grow exponentially with respect to the radius around a given point. However, as long as the radius of a sequence of manifolds does not grow with the ambient dimension, the Bishop-Gromov inequality bounds the volume of such manifolds rendering them efficiently sampleable. As far as we are aware, estimating the diameter of real-world manifolds is a challenging task and it is unknown 
whether real-world data manifolds can fit within this regime.

\begin{example}[Manifolds with bounded curvature and radius] \label{ex:negative_curvature}
    Any sequence of distributions $\{\mathcal{P}_{\mathcal{M}_n}\}$ supported over manifolds $\{\mathcal{M}_n\}$ where each manifold $\mathcal{M}_n$ has bounded Ricci curvature $\Ric{\mathcal{M}_n} \geq (d-1)K$ for an arbitrary constant $K$ and is contained within a ball of radius $r=O_n(\log(n))$ around a point $\vp_n \in \mathcal{M}_n$ is efficiently sampleable.
\end{example}
\begin{proof}
    By the Bishop-Gromov theorem (\Cref{thm:bishop_gromov}), the volume of any such manifold is bounded by $\Vol_{S_d^K}(r)$ which is the volume of a ball of radius $r$ in the $d$-dimensional Hyperbolic space of constant curvature $K$. This means that 
    \begin{equation}
        \Vol_d(\mathcal{M}_n) \leq  \Vol_{S_d^K}(r) = \sigma_{d-1} \int_{t=0}^{r} \sinh^{d-1}(t) \; dt \leq \sigma_{d-1} r \left( \frac{\exp(r)}{2} \right)^{d-1}.
    \end{equation}
    Thus, whenever $r=O_n(\log(n))$, this allows for an $\epsilon$-cover of size $O_n(\poly(n))$. The rest of the proof follows directly from that of \Cref{ex:isoperimetric_setting}.
\end{proof}

Another example of a class of manifolds that fall in the learnable regime are those which are $\alpha$-strongly convex \cite{scieur2023strong}, a property that can be leveraged in instances of Riemannian optimization. This property is extended from the standard notion of strong convexity on Hilbert spaces.

\begin{definition}[$\alpha$-strong convexity of a Hilbert space \cite{goncharov2017strong}]\label{def:alpha_strong_convex_euclidean}
    Given a real Hilbert space $H$ with inner product $\langle \cdot, \cdot \rangle$ and corresponding induced norm $\| \cdot \|$, denote by $B_R(c)$ the closed ball of radius $R$ centered around $c \in H$. A subset $A \subset H$ is $\alpha$-strongly convex if there exists a set $C \subset H$ such that
    \begin{equation}
        A = \bigcap_{c \in C} B_R(c),
    \end{equation}
    where $R = \frac{1}{2\alpha}$.
\end{definition}

The above definition captures the equivalent notion in the Euclidean space $\mathbb{R}^d$ that an $\alpha$-strongly convex set $A \subset \mathbb{R}^d$ is one where for any $\vx,\vy \in A$ and unit norm $\vz \in \mathbb{R}^d$ such that $\|\vz \| = 1$, we have that \cite{scieur2023strong}
\begin{equation}
    (1-t)\vx+ t \vy + \alpha (1-t)t \|\vx - \vy\|^2 \vz \in A.
\end{equation}

Riemannian manifolds are $\alpha$-strongly convex if the image of the exponential map over inputs in a set $A$ living in the tangent space are $\alpha$-strongly convex in the Euclidean sense.

\begin{definition}[$\alpha$-strong convexity of a Riemannian manifold \cite{scieur2023strong}]\label{def:riem_alpha_strong_convex}
    Let $\mathcal{M}$ be a Riemannian manifold that is uniquely geodesic (i.e. any two points $\vx, \vy \in \mathcal{M}$ can be connected by a unique geodesic). Then, $\mathcal{M}$ is Riemannian $\alpha$-strongly convex if for any $\vx \in \mathcal{M}$, the set
    \begin{equation}
        \operatorname{Exp}^{-1}_\vx(\mathcal{M}) \coloneqq \left\{ \vy \in T_{\vx}\mathcal{M}:  \vz = \operatorname{Exp}_\vx(\vy), \vz \in \mathcal{M}  \right\}
    \end{equation}
    is $\alpha$-strongly convex with respect to the inner product $\|\cdot\|_\vx$ on $T_{\vx}\mathcal{M}$ in the sense of \Cref{def:alpha_strong_convex_euclidean}.
\end{definition}

The above definition places restrictions on the diameter of manifolds, essentially placing them in the setting of bounded radius manifolds studied in \Cref{ex:negative_curvature}.

\begin{example}[Strongly convex manifolds]
    As stated in \Cref{def:alpha_strong_convex_euclidean} for any $\alpha$-strongly convex manifold $\mathcal{M}$, the diameter of $\operatorname{Exp}^{-1}_\vx(\mathcal{M})$ is at most $2r=\alpha^{-1}$. For a sequence of such Riemannian $\alpha$-strongly convex manifolds $\{\mathcal{M}_n\}$, as long as the metric over the manifold in the normal coordinates at $T_\vx \mathcal{M}$ is bounded everywhere (or say the curvature is bounded), then such sequences of manifolds are efficiently sampleable similar to the setting of \Cref{ex:negative_curvature}.
\end{example}

\section{Deferred Proofs}\label{app:proofs_for_bounded_curvature}

\subsection{Space-filling manifold}
\label{app:space_filling_curve}

To construct a manifold which achieves the desired lower bound, we form a customized manifold which acts as a space-filling curve touching exponentially many quadrants of the $n$-dimensional hypercube. This space-filling curve is constructed by maneuvering around a Gray code which enumerates Boolean strings in such a fashion that successive strings differ by Hamming distance one. This procedure is motivated by and reminiscent of techniques used in constructing Hilbert curves from Gray codes \cite{skilling2004programming}.

\begin{lemma}[Gray code \cite{gray1953pulse}] \label{def:gray_code}
    For every integer $k>0$, there exists a bijective function $G:[2^{k}] \to \{0,1\}^{k}$ enumerating length $k$ bitstrings such that successive bitstrings $G(i)$ and $G(i+1)$ differ in only one coordinate, i.e., $|G(i) - G(i+1)|_H = 1$ where $|\cdot|_H$ denotes the Hamming distance.
\end{lemma}

As a simple example, for $k=3$, the sequence $[000,001,011,010,110,111,101,100]$ is one such Gray code which can be constructed recursively as a binary-reflected Gray code \cite{knuth1997art}. We also visualize the $k=7$ bit Gray code in \Cref{fig:gray-code}.

\begin{figure}
    \centering
    \includegraphics[width=\textwidth]{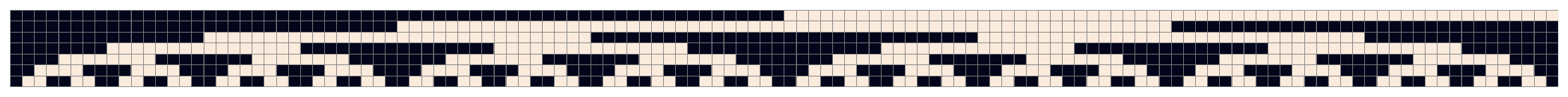}
    \caption{Enumeration of Gray code for $n=7$ bits. Each column corresponds to a bitstring where black square is equal to $0$ and tan square is $1$. Note that the variation takes place largely in the last entries (bottom-most).}
    \label{fig:gray-code}
\end{figure}

Given an intrinsic dimension $d$, we will construct manifolds $\mathcal{M}$ as products of $P_d = \{\vx \in [0,1]^{d-1}\}$, the $(d-1)$-dimensional hypercube, and a one-dimensional submanifold over the remaining dimensions, which resembles a one-dimensional space-filling curve. The submanifold is constructed to touch the corners of the hypercube following a Gray code.

\paragraph{One-dimensional space-filling curve construction.}
Our goal is to construct a one-dimensional sub-manifold $\mathcal{M} \subset [0,1]^{n-d+1}$ that covers as many quadrants of the hypercube as possible. Given a radius of curvature $R$, which we will later set to be within the bounds of the stated reach, let $n_R = \floor{(n-d+1)/\delta_R}$ and $\delta_R = \ceil{4R^2}$. In our construction, we will use the Gray code on $n_R$ bits and map these to bitstrings in dimension $n$ by copying the bitstring of the Gray code $\delta_R$ times. This construction will allow the manifold to have radius of curvature conforming to the given bound on the reach.

\begin{figure}
    \centering
    \includegraphics{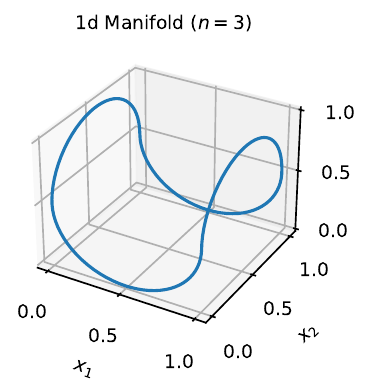}
    \caption{Shape of constructed manifold $\mathcal{M}_3$.}
    \label{fig:M1-shape}
\end{figure}

Specifically, let $G:[2^{n_R}] \to \{0,1\}^{n_R}$ be a Gray code over $n_R$ bits, which maps indices $1, \dots, 2^{n_R}$ to the corresponding bitstring in the Gray code (see \Cref{def:gray_code}). Then, for $i \in [2^{n_R}]$ let $\vb_i = G(i)^{\oplus \delta_R} \oplus \bm 0_{n - n_R}$ where $G(i)^{\oplus \delta_R}$ is the bitstring $G(i)$ repeated $\delta_R$ times and $\bm 0_{k}$ is the bitstring of length $k$ with each entry equal to $0$. $\bm 0_{k}$ is concatenated at the end to ensure the input is of dimension $n$. Then, the space-filling sub-manifold takes the form:
\begin{equation} \label{eq:space_filling_manifold_equation}
    \mathcal{M}_{n_R} = \bigcup_{k=1}^{2^{n_R}} \left\{  \frac{\vb_{k-1} + \vb_{k+1}}{2} + \left( \frac{\vb_{k} - \vb_{k+1}}{2}\right) \cos(t) +  \left( \frac{\vb_{k} - \vb_{k-1}}{2}\right) \sin(t) : t \in [0, \pi / 2] \right\}.
\end{equation}
As an example, the manifold $\mathcal{M}_3$ is visualized in \Cref{fig:M1-shape}.
Indexing of the bitstrings $\vb_k$ is assumed circular and taken $k \mod 2^n$.

\begin{lemma} \label{lem:manifold_reach_verification}
    The manifold $\mathcal{M}_{n_R}$ defined in \Cref{eq:space_filling_manifold_equation} has $\Reach{\mathcal{M}_{n_R}} = R$. Furthermore, defining 
    \begin{equation}
        \operatorname{Round}(\mathcal{M}_{n_R}) = \left\{ \frac{1}{2}\left(\sgn(\vx - \vone_{n_R} / 2) + \vone_{n_R} \right): \vx \in \mathcal{M}_{n_R} \right\}
    \end{equation}
    as the operation which projects each point onto the nearest corner corresponding to bitstrings of the hypercube, we have that $|\operatorname{Round}(\mathcal{M}_{n_R})| = 2^{n_R}$.
\end{lemma}

\begin{proof}
    Each segment of the manifold in \Cref{eq:space_filling_manifold_equation} is an arc of a circle centered at $\frac{\vb_{k-1} + \vb_{k+1}}{2}$ and of radius $\sqrt{\delta_R}/2$. For two segments indexed by $ k'$ such that $|k-k'|\geq 2$, the segments differ maximally in at least $\delta_R$ locations so the medial axis between these segments is at least $\sqrt{\delta_R}/2 \geq R$ distance away. 
    
    Between neighboring segments indexed by $k$ and $k+1$, we note that we have a curve which consists of circular arcs from $[\bm 0_{\delta_R}, \vone_{\delta_R}/2, \vone_{\delta_R}] \oplus \vs $ to $[\bm 0_{\delta_R}, \vone_{\delta_R}, \vone_{\delta_R}/2] \oplus \vs $ and then to $[\vone_{\delta_R}/2,  \vone_{\delta_R}, \bm 0_{\delta_R}] \oplus \vs $ for some bitstring $\vs$ which is shared between all of the points in the curve. Here, for simplicity, we have permuted the distinct elements in the neighboring segments to the beginning of the bitstrings. Consider now the medial axis of this curve and set $\delta_R=1$ for simplicity as the general case will directly follow. Here, the medial axis takes the form of points $[a,1-b,c] \oplus \vs$ where $0 \leq a,b \leq 1/2$ and $0 \leq c \leq 1$. For any such point where $c>0.5$ or $c<0.5$, the nearest point in the manifold will be a point in the first or second segment respectively. For such points, the medial axis will thus be the center of the corresponding circle which the arc spans. When $c=0.5$, if $a=b<0.5$, then the unique nearest point is always $[0,1,1/2] \oplus \vs$. Otherwise, if $c=0.5$ and $a<b$ or $b<a$, the nearest point in the manifold will be a point in the first or second segment respectively as in the previous case. Putting this all together, the reach of this manifold is $\sqrt{\delta_R}/2 \geq R$. 

    Finally, to count the number of elements in $\operatorname{Round}(\mathcal{M}_{n_R})$, note that setting $t=\pi/4$ for segment $k$ in \Cref{eq:space_filling_manifold_equation} obtains a point of the form 
    \begin{equation}
        \left(\vb_{k-1} + \vb_{k+1}\right)\frac{1}{2}\left(1 - \frac{\sqrt{2}}{2}\right) + \frac{\sqrt{2}}{2}\vb_k.
    \end{equation}
    The above is rounded to $\vb_k$. Applying this procedure for all $k$ gives $2^{n_R}$ total rounded points. 
\end{proof}

\begin{corollary}
    For a sequence of manifolds $\{\mathcal{M}_n\}$ of intrinsic dimension $d=O(1)$ and reach $\Reach{\mathcal{M}_n}=O(n^{\alpha})$ for $\alpha < 0.5$ taking the form of \Cref{lem:manifold_reach_verification}, the manifold covers exponentially many quadrants of the Boolean cube, i.e. $|\operatorname{Round}(\mathcal{M}_{n})| = 2^{\Omega(n)}$.
\end{corollary}

In our proofs, we will consider the setting where we have a sequence of manifolds of increasing intrinsic dimension $n$ and bounded reach. We choose the parameters of this sequence as follows.
\begin{definition}[$(\alpha, d)$-sequence of manifolds] \label{def:manifold_sequence}
    For given intrinsic dimension $d$ and bound on the reach $R = O(n^{\alpha})$ for some $\alpha<0.5$, set $n = \ceil{4R^2}n_b+d-1$. We construct a sequence of manifolds $\{\mathcal{M}_n\}$ of the form given in \Cref{eq:space_filling_manifold_equation} with the given values of $R,n,d$ where $n$ is incremented by increasing values $n_b$.
\end{definition}

\subsection{SQ hardness}
We present here the hardness results in the statistical query (SQ) model by mapping a class of hard to learn Boolean functions to real-valued neural network functions, which approximate those Boolean functions well. This follows a commonly used set of proof techniques employed in \cite{diakonikolas2020algorithms,daniely2021local,chen2022hardness,goel2020superpolynomial}. We follow most closely the line of reasoning in \cite{chen2022hardness} for showing hardness results for learning two hidden layer $\operatorname{ReLU}$ networks though the nature of the distribution in our case will allow us to show SQ hardness results for single hidden layer networks. Similar hardness results based on reductions to cryptographically hard problems as shown in the work of \cite{daniely2021local} are also provided in \Cref{app:crypto_equivalent_hardness}.

Hardness in our setting is based on classic results of the SQ hardness of learning parity functions over Boolean inputs \cite{kearns1998efficient,blum1994weakly}. A parity function $\chi_S:\{0,1\}^d \to \{0,1\}$ is a Boolean function taking the form
\begin{equation}
    \chi_S(\vx_b) = \sum_{i \in S} [\vx_b]_i \mod 2,
\end{equation}
which determines whether the summation of the input Boolean string over indices in the set $S \subseteq [d]$ is even or odd. The class of such parity functions indexed by subsets $S \subseteq [d]$ is exponentially hard to learn in the SQ model.
\begin{theorem}[SQ hardness of parities \cite{blum1994weakly,kearns1998efficient}, see also Theorem 4.3 of \cite{chen2022hardness}] \label{thm:parity_SQ_lower_bound}
    Any SQ algorithm given SQ access to the distribution of labeled pairs $(\vx_b, y)$ where $\vx_b \sim \operatorname{Unif}(\{0,1\}^d)$ and $y = \chi_S(\vx_b)$ for an unknown $S \subseteq [d]$ capable of learning $\chi_S$ up to classification error $\epsilon$ sufficiently small with queries of tolerance $\tau$ requires $\Omega(\tau^2 2^d)$ queries.
\end{theorem}

For Boolean inputs $\vx_b \in \{0,1\}^d$, the parity function $\chi_S(\vx_b)$ can be constructed as a single hidden layer neural network with $O(d)$ nodes. Ideally, one would want to replicate this over real-valued inputs, but since ReLU networks are continuous, this construction cannot be made exactly. \cite{daniely2021local} show that, in practice, one can form a network which is equivalent to the $\sgn$ function for all but a small percentage of inputs. Combined with an indicator function that zeros out the output wherever such inputs cannot be rounded exactly, they show that equivalent hardness results can be obtained in real-valued settings. This construction was extended to SQ settings by \cite{chen2022hardness}. 

In our setting where the underlying distribution is uniformly drawn from the space-filling manifold described in \Cref{app:space_filling_curve}, we can resort to a simpler construction given that most input values are either $0$ or $1$ at any point in the space-filling manifold. This lets us identify approximate parity functions which for all but an exponentially small fraction of inputs is equivalent to the Boolean parity function. 

To achieve this goal, we need to map Boolean parity functions $\chi_S:\{0,1\}^{n_b} \to \{0,1\}$ to real-valued neural network functions $f:\mathbb{R}^n \to \mathbb{R}$ that approximate the parity functions appropriately. In this construction, for given intrinsic dimension $d$ and bound on the reach $R = O(n^{\alpha})$ for some $\alpha<0.5$, we set $n = \ceil{4R^2}n_b+d-1$ in the construction given in \Cref{app:space_filling_curve} ($n$ chosen so that the number of bits in the Gray code $n_R$ is equal to $n_b$). Each input entry is repeated $\ceil{4R^2}$ times in this construction, so for a given input $\vx \in \mathbb{R}^n$, we define $\mathcal{P}:\mathbb{R}^{n} \to \mathbb{R}^{n_b}$ as an operation, which selects one of the $\ceil{4R^2}$ many equivalent inputs for each of the $n_b$ potentially unique values. That is, for segments of the manifold taking the form of \Cref{eq:space_filling_manifold_equation}, each segment is a sum of elements $\vb_i = G(i)^{\oplus \delta_b} \oplus \bm 0_{n - n_b}$ where $G:[2^{n_b}] \to \{0,1\}^{n_b}$ is a Gray code over $n_b$ bits. Setting $\mathcal{P} \vx = [\vx]_{:n_b}$ to capture the first $n_b$ elements then suffices to perform such an operation.

The intuition for our construction is as follows: We have to show that at most points on the manifold, inputs are with high probability equal to some Boolean string. If we choose the marginal distribution over the first $k$ input locations of $\mathcal{P}\vx$ with $k<n_b$ sufficiently smaller than $n_b$, then with high probability, these will always be constantly set to either value $0$ or $1$ in the segments of the manifold (\Cref{eq:space_filling_manifold_equation}). This is because the reflected Gray code is formed by changing the right-most bits first whenever possible so the left-most bits are constant along most segments. We formalize this below.  
\begin{lemma}\label{lem:high_prob_boolean_string}
    Let $\mathcal{D}_{\mathcal{M}_n}$ denote the uniform distribution over the $(\alpha,d)$-sequence of manifolds $\mathcal{M}_n$ constructed in \Cref{def:manifold_sequence} for given bounds on the reach $R=O(n^\alpha)$ and intrinsic dimension $d$. Given $\vx \in \mathcal{M}_n$ constructed over the binary code over $n_b$ bits, let $[\mathcal{P}\vx]_{:n_b-t}$ denote the vector consisting of the first $n_b-t$ entries of $\mathcal{P}\vx$ for $t \in [n_b]$. Then, with probability $1-O(2^{-t})$, a draw of $\vx \sim \mathcal{D}_{\mathcal{M}_n}$ will have the property that $[\mathcal{P}\vx]_{:n_b-t} \in \{0,1\}^{n_b-t}$. Additionally, with probability $2^{t-n_b} + O(2^{-t})$ over independent draws $\vx, \vx' \sim \mathcal{D}_{\mathcal{M}_n}$, $[\mathcal{P}\vx]_{:n_b-t} \neq [\mathcal{P}\vx']_{:n_b-t}$ and the resulting values will differ in at least one location.
\end{lemma}
\begin{proof}
    In the binary reflected Gray code $G:[2^{n_b}] \to \{0,1\}^{n_b}$ over $n_b$ bits, at any location $i \in [2^{n_b}]$, there exists a span of integers $[a,b)$ containing $i$ of length $b-a=2^t$ such that $[G(j)]_{n_b-t}=[G(j')]_{n_b-t}$ is equal for all $j,j' \in [a,b) $. Furthermore, by removing the endpoints where for any $j,j' \in [a+1,b-1)$, any entry of $[G(j)]_{k}$ for $k \in [n_b-t]$ will also have the property that $[G(j)]_{k}=[G(j')]_{k}$ (see \Cref{fig:gray-code} for example for a visualization of this property). 
    The value at a given point in a segment of the manifolds in \Cref{eq:space_filling_manifold_equation} is not equal to zero or one only when bitstring indices are differing in three adjacent points in the Gray code. Given a random point $\vx \sim \mathcal{D}_{\mathcal{M}_n}$, it will fall within a given segment of the Gray code spanned over elements $i-1, i, i+1$ of the Gray code for $i \sim \operatorname{Unif}([2^{n_b}])$. Thus, the probability that $[\mathcal{P}\vx]_{:n_b-t} \in \{0,1\}^{n_b-t}$ is at least $1-O(2^{-t})$. 

    To prove the additional fact, note that for any $\vz \in \{0,1\}^{n_b-t}$, we have that
    \begin{equation} \label{eq:upper_bound_on_prob}
        \mathbb{P}[[\mathcal{P}\vx]_{:n_b-t} = \vz] \geq \frac{2^t-8}{2^{n_b}}.
    \end{equation}
    This follows from the previously stated property that at any location $i \in [2^{n_b}]$ in the gray code, there exists a span of integers $[a,b)$ containing $i$ of length $b-a=2^t$ such that $[G(j)]_{k}=[G(j')]_{k}$ is equal for all $j,j' \in [a+1,b-1) $ and all $k \in [n_b-t]$. From this span $[a,b)$ we remove the first and last three indices to get the span $[a+4,b-4)$ over which $[G(i)]_{n_b-t} = [G(i+k)]_{n_b-t} $ for all $i \in [a+4,b-4)$ and $k \in [-2,2]$. I.e. this span is chosen to ensure that $\vx_{n_b-t}$ is constant in \Cref{eq:space_filling_manifold_equation} over those segments of the Gray code. From here, we have that over independent draws $\vx, \vx' \sim \mathcal{D}_{\mathcal{M}_n}$:
    \begin{equation}
    \begin{split}
        \mathbb{P}_{\vx, \vx' \sim \mathcal{D}_{\mathcal{M}_n}}\left[ [\mathcal{P}\vx]_{:n_b-t} 
        = [\mathcal{P}\vx']_{:n_b-t}, [\mathcal{P}\vx]_{:n_b-t} \in \{0,1\}^{n_b-t} \right] &= \sum_{\vz \in \{0,1\}^{n_b-t}} \mathbb{P}\left[[\mathcal{P}\vx]_{:n_b-t} = \vz \right]^2 \\
        &\leq \max_{\vz \in \{0,1\}^{n_b-t}} \mathbb{P}\left[[\mathcal{P}\vx]_{:n_b-t} = \vz \right] \\
        &\leq 1 - (2^{n_b-t}-1)\frac{2^t-8}{2^{n_b}} \\
        &\leq 2^{t-n_b}+8(2^{-t}).
    \end{split}
    \end{equation}
    The second line above follows from Hölder's inequality. The third line above applies \Cref{eq:upper_bound_on_prob}. Noting that the last term is $2^{t-n_b} + O(2^{-t})$ completes the proof.
\end{proof}

The above shows that with exponentially high probability, samples drawn from the distribution over the space-filling manifold will be equivalent to those drawn uniformly over Boolean inputs. This lets us complete the proof of hardness.

\begin{theorem}
    Let $\mathcal{D}_{\mathcal{M}_n}$ denote the uniform distribution over the $(\alpha, d)$ sequence of manifolds $\mathcal{M}_n$ constructed in \Cref{def:manifold_sequence} where $\Reach{\mathcal{M}_n} = O(n^\alpha)$ for $\alpha<0.5$. 
    Any SQ algorithm $\mathcal{A}$ capable of learning the class of linear width single hidden layer $\operatorname{ReLU}$ neural networks under this sequence of distributions up to mean squared error sufficiently small ($\epsilon/8$ suffices) with queries of tolerance $\tau$ must use at least $\Omega(\tau^2 2^{n^{\Omega(1)}})$ queries. 
\end{theorem}
\begin{proof}
    Given the sequence of manifolds $\{\mathcal{M}_n\}$ from the construction in \Cref{def:manifold_sequence}, we have $n_b = \Omega(n^{1-2\alpha})$ which is chosen to fit the bounds on the reach of the manifold. 

    Given a parity function $\chi_S:\{0,1\}^{d}\to \{0,1\}$ taking the form
    \begin{equation}
        \chi_S(\vx_b) = \sum_{i \in S} [\vx_b]_i \mod 2,
    \end{equation}
    consider its continuous approximation $\widetilde{\chi}_S:[0,1]^{d} \to \{0,1\}$ defined as
    \begin{equation}
        \widetilde{\chi}_S(\vx)=
        \begin{cases}
                    \sum_{i \in S} \vx_i - k & \sum_{i \in S} \vx_i \in [k, k+1], k \in \{0,2,4,...,d\}\\
                    k + 1 - \sum_{i \in S} \vx_i & \sum_{i \in S} \vx_i \in [k, k+1], k \in \{1,3,5,...,d-1\}
                \end{cases} \; .
    \end{equation}
    Note, that the above is valid for $d$ even and a similar equation can be obtained for $d$ odd. The above is also piecewise linear and can be constructed as a single hidden layer $\operatorname{ReLU}$ network with $O(d)$ width. We consider learning the class $\mathcal{C}$ of single hidden layer networks consisting of
    \begin{equation}
        \mathcal{C} = \left\{f_S: S \subseteq [n_b-t] \right\} \text{ where } f_S(\vx) = \widetilde{\chi}_S\left( [\mathcal{P}\vx]_{:n_b-t} \right).
    \end{equation}
    Setting $t = n_b/2$ suffices by \Cref{lem:high_prob_boolean_string} to guarantee that $[\mathcal{P}\vx]_{:n_b/2} $ is a Boolean string with probability $1-O(2^{-\Omega(n_b)})$. Furthermore, for $\vx, \vx' $ drawn i.i.d. from $\mathcal{D}_{\mathcal{M}_n}$, $[\mathcal{P}\vx]_{:n_b/2} \neq [\mathcal{P}\vx']_{:n_b/2} $ with probability $1-O(2^{-\Omega(n_b)})$. Conditioned on this event, the distribution of $(f_S(\vx), f_S(\vx'))$ is equal to $\operatorname{Unif}(\{0,1\}^2)$. Since $n_b = \Omega(n^{1-2\alpha})$, we can apply \Cref{thm:pairwise-ind-sq-lower} noting that learning the function class up to MSE $\epsilon$ sufficiently small suffices to distinguish the distribution task in \Cref{thm:pairwise-ind-sq-lower}.
\end{proof}

\subsection{Cryptographic hardness reduction}
\label{app:crypto_equivalent_hardness}

Here, we will prove hardness results for learning single hidden layer networks under cryptographic hardness assumptions following the methodology in \cite{daniely2021local}. First, we detail the construction of Goldreich's pseudorandom generator (PRG) \cite{goldreich2011candidate} which are used as the basis for the hard to learn class of single hidden layer neural networks in \cite{daniely2021local}. In fact, \cite{daniely2021local} show that given random Boolean inputs, a neural network with width $\omega(1)$ suffices to construct these hard to learn functions. From here, we show that the manifold constructed in \Cref{sec:manifold_reconstruction} can produce inputs that are approximately uniformly distributed over the Boolean cube when restricted to a certain set of input locations. This lets us directly apply the results of \cite{daniely2021local} to this distribution.

\paragraph{Pseudorandomness assumption}
The pseudorandom functions are constructed over an $(n,m,k)$-hypergraph with $n$ vertices $[n]$ with $m$ ordered hyperedges $S_1,\ldots,S_m$, each having cardinality $k$ and no repeated entries. Hypergraphs are drawn from the Erdős–Rényi distribution $\mathcal{G}_{n,m,k}$ where each hyperedge is drawn i.i.d. from the set of $n!/(n-k)!$ possible hyperedges (equivalent to ordered sets).

\begin{definition}[Goldreich's pseudorandom generator (PRG) \cite{goldreich2011candidate}]
    Given constant integer $k>0$, predicate $P:\{0,1\}^k \rightarrow \{0,1\}$, and a $(n,m,k)$-hypergraph $G$, Goldreich's pseudorandom generator (PRG) is the function $f_{P,G}: \{0,1\}^n \rightarrow \{0,1\}^m$, such that given input $\vx \in \{0,1\}^n$ returns $f_{P,G}(\vx) = (P(\vx_{S_1}),\ldots,P(\vx_{S_m}))$. The PRG has polynomial stretch if $m=n^a$ for some $a>1$.
\end{definition}

Denoting the collection of functions $f_{P,G}$ over $(n,m,k)$-hypergraphs as $\mathcal{F}_{P,n,m}$ then $\mathcal{F}_{P,n,m}$ is an $\varepsilon$-pseudorandom generator ($\varepsilon$-PRG) if every polynomial-time probabilistic algorithm $\mathcal{A}$ has advantage at most $\epsilon$ in distinguishing pairs $(G, f_{P,G}(\vx))$ from random pairs $(G, \vy)$ where $\vy$ is drawn i.i.d. from the uniform distribution on $\{0,1\}^m$. This is formalized below.

\begin{definition}[$\varepsilon$-pseudorandom generator ($\varepsilon$-PRG) \cite{daniely2021local}]
    $\mathcal{F}_{P,n,m}$ is an $\varepsilon$-pseudorandom generator ($\varepsilon$-PRG) if for every polynomial-time probabilistic algorithm $\mathcal{A}$, it holds that
    \begin{equation}
        \left| \mathbb{P}_{G \sim \mathcal{G}_{n,m,k}, \vx \sim \mathcal{B}_n}\left[\mathcal{A}(G, f_{P,G}(\vx)) = 1 \right] - \mathbb{P}_{G \sim \mathcal{G}_{n,m,k}, \vy \sim \mathcal{B}_m}\left[\mathcal{A}(G, \vy) = 1 \right] \right| \leq \epsilon,
    \end{equation}
    where $\mathcal{B}_{k}$ denotes the uniform distribution over bitstrings $\{0,1\}^k$.
\end{definition}

The cryptographic hardness assumption posits that there does exist such a class as above for which it is hard to distinguish between the cases where $\vy$ is completely random or the output of a randomly drawn function $f_{P,G}$.

\begin{assumption}[Existence of $\epsilon$-PRG \cite{daniely2021local}]
\label{ass:localPRG}
For every constant $s>1$, there exists a constant $k$ and a predicate $P:\{0,1\}^k \rightarrow \{0,1\}$, such that $\mathcal{F}_{P,n,n^s}$ is $\frac{1}{3}$-PRG.
\end{assumption}

\paragraph{Hard to learn neural network function class}
We follow the procedure in the previous subsection to map hard to learn Boolean functions to the task of learning real-valued functions where inputs are drawn from the data manifold. In this procedure, we must show that instances where a given input does not correspond to a Boolean string occur with a vanishingly small probability. Following the previous construction, for given intrinsic dimension $d$ and bound on the reach $R = O(n^{\alpha})$ for some $\alpha<0.5$, we again set $n = \ceil{4R^2}n_b+d-1$ in the construction given in \Cref{app:space_filling_curve} ($n$ chosen so that the number of bits in the Gray code $n_R$ is equal to $n_b$). As each input entry is repeated $\ceil{4R^2}$ times in this construction, we recall the definition of $\mathcal{P}:\mathbb{R}^{n} \to \mathbb{R}^{n_b}$ selecting one of the $\ceil{4R^2}$ many equivalent inputs for each of the $n_b$ potentially unique values. The following lemma then guarantees upon that taking the first $n_b-t$ entries, with probability $1-O(2^{-t})$ the distribution of $[\mathcal{P}\vx]_{:n_b-t}$ will be equivalent to the uniform distribution over Boolean strings.

\begin{lemma}\label{lem:high_prob_boolean_uniform_distribution}
    Let $\mathcal{D}_{\mathcal{M}_n}$ denote the uniform distribution over the $(\alpha,d)$-sequence of manifolds $\mathcal{M}_n$ constructed in \Cref{def:manifold_sequence} for given bounds on the reach $R=O(n^\alpha)$ and intrinsic dimension $d$. Given $\vx \in \mathcal{M}_n$ constructed over the binary code over $n_b$ bits, let $[\mathcal{P}\vx]_{:n_b-t}$ denote the vector consisting of the first $n_b-t$ entries of $\mathcal{P}\vx$ for $t \in [n_b]$. Denoting $\mathcal{P}\mathcal{D}_{\mathcal{M}_n}$ as the distribution of $[\mathcal{P}\vx]_{:n_b-t}$ for $\vx \sim \mathcal{D}_{\mathcal{M}_n}$, we have that with probability at least $1-2^{3-t}$, $\mathcal{P}\mathcal{D}_{\mathcal{M}_n}$ returns a uniformly random Boolean string:
    \begin{equation}
        \mathcal{P}\mathcal{D}_{\mathcal{M}_n} \sim \begin{cases}
            \mathcal{B}_{n_b-t} & \text{w.p. } 1-2^{3-t}, \\
            \mathcal{D}_{\rm rem} & \text{w.p. } 2^{3-t},
        \end{cases}
    \end{equation}
    where $\mathcal{B}_{k}$ denotes the uniform distribution over bitstrings $\{0,1\}^k$ and $\mathcal{D}_{\rm rem}$ is an arbitrary distribution handling the event where draws from the distribution $\mathcal{P}\mathcal{D}_{\mathcal{M}_n}$ are not uniform over the Boolean string.
\end{lemma}
\begin{proof}
    It suffices to show that there exists an event $E$ occuring with probability at least $1-2^{3-t}$ such that conditioned on this event, for every $\vz \in \{0,1\}^{n_b-t}$:
    \begin{equation}
        \mathbb{P}\left[[\mathcal{P}\vx]_{:n_b-t} = \vz | E\right] = 2^{t-n_b},
    \end{equation}
    i.e. every bitstring $\vz \in \{0,1\}^{n_b-t}$ is returned with equal probability. In fact, previously in \Cref{lem:high_prob_boolean_string} and specifically \Cref{eq:upper_bound_on_prob}, we showed that
    \begin{equation} 
        \mathbb{P}[[\mathcal{P}\vx]_{:n_b-t} = \vz] \geq \frac{2^t-8}{2^{n_b}}.
    \end{equation}
    Therefore, we can construct the event $E$ by taking the union over the bitstrings $\vz$, each of probability $\frac{2^t-8}{2^{n_b}}$ ignoring a portion of the distribution whenever a given $\mathbb{P}[[\mathcal{P}\vx]_{:n_b-t} = \vz]$ exceeds this value. This implies that
    \begin{equation}
        \mathbb{P}[E] \geq 2^{n_b-t}\left( \frac{2^t-8}{2^{n_b}}\right) = 1-2^{3-t}.
    \end{equation}
\end{proof}

\begin{theorem}[Cryptographic hardness] \label{thm:cryptographic}
    Let $\mathcal{D}_{\mathcal{M}_n}$ denote the uniform distribution over manifolds $\mathcal{M}_n$ constructed in \Cref{def:manifold_sequence} where $\Reach{\mathcal{M}_n} = O(n^\alpha)$ for $\alpha<0.5$. 
    Set $a>0$ to be constant. Then under \Cref{ass:localPRG}, there does not exist any polynomial time algorithm $\mathcal{A}$ that with probability $1-\delta$ for $0<\delta<1/2$ is capable of learning the class of width $O(n^a)$ single hidden layer $\operatorname{ReLU}$ neural networks under this sequence of distributions up to mean squared error sufficiently small ($\epsilon/40$ suffices).
\end{theorem}
\begin{proof}
    Theorem 8 (part 3) of \cite{daniely2021local} states that under \Cref{ass:localPRG}, there does not exist an efficient algorithm to learn the class of $O(n^a)$ width neural networks under the uniform input distribution $\mathcal{B}_m$ over bitstrings $\{0,1\}^m$. Namely, they show that there exist a set of neural network functions $f:\mathbb{R}^m \to [0,1]$ that have the property that they recover a hard to learn Boolean DNF formula $f(\vx) \in \{0,1\}$ for any $\vx \in \{0,1\}^m$. They then show that any algorithm which with probability $1-\delta$ returns a hypothesis $h_b:\{0,1\}^m \to \{0,1\}$ achieving classification error $\mathbb{P}_{\vx \sim \mathcal{B}_m}\left[h_b(\vx) \neq f(\vx) \right] \leq 1/10$ suffices to contradict \Cref{ass:localPRG}.

    Any efficient algorithm can only take $O(\poly(n))$ samples. From \Cref{lem:high_prob_boolean_uniform_distribution}, we can set $m = n_b/2$ with $t=n_b/2$ so that with probability at least $1-\exp(-n^{\Omega(1)})$ a given sample is drawn from $\mathcal{B}_{n_b/2}$. By union bounding over $O(\poly(n))$ samples, we are guaranteed that all samples are drawn i.i.d. from $\mathcal{B}_{n_b/2}$ with probability $1-\delta'$ for $\delta'$ vanishing exponentially in $n$. Given any input $\vx_b \in \{0,1\}^m$, we can convert it to a random point on the corresponding quadrant in the manifold $\vx = \mathcal{T}\vx_b \in \mathcal{M}_n$ by drawing $t \sim \operatorname{Unif}([0,\pi/2])$ and applying \Cref{eq:space_filling_manifold_equation}. This recovers the distribution $\mathcal{D}_{\mathcal{M}_n}$. Now, assume that the algorithm for learning single hidden layer neural networks under the distribution $\mathcal{D}_{\mathcal{M}_n}$ returns a hypothesis $h:\mathbb{R}^n \to \mathbb{R}$. We can convert this to a boolean-output function $h':\mathbb{R}^m \to \{0,1\}$ by setting $h'(\vx)=\sgn(h(\vx)-1/2)$. Note that (see also Lemma 23 of \cite{daniely2021local}) whenever $h'(\mathcal{T}\vx_b) \neq f(\vx_b)$ then $(h(\mathcal{T}\vx_b) \neq f(\vx_b))^2\geq 1/4$ and
    \begin{equation}
        \mathbb{P}_{\vx_b \sim \mathcal{B}_m,t \sim \operatorname{Unif}([0,\pi/2]) }\left[ h'(\mathcal{T}\vx_b) \neq f(\vx_b) \right] \leq 4 \mathbb{E}_{\vx_b \sim \mathcal{B}_m,t \sim \operatorname{Unif}([0,\pi/2])}\left[(h(\mathcal{T}\vx_b) - f(\vx_b))^2 \right].
    \end{equation}
    The right hand side above contains the mean squared error loss over random Boolean inputs. This is equivalent to the mean squared error loss under the distribution $\mathcal{D}_{\mathcal{M}_n}$ up to an exponentially small correction for the event where $[\mathcal{P}\vx]_{:n_b-t}$ is not Boolean. Thus, learning the function $f$ up to mean squared error $\epsilon/40$ suffices to output a function $h'$ achieving classification error $\epsilon/10$ (up to these exponentially vanishing corrections). Applying Theorem 8 of \cite{daniely2021local} finishes the proof.
\end{proof}

\subsection{Volume bounds based on reach}
\label{app:volume_bound_reach}

The reach of a manifold lower bounds the radius of curvature of a manifold. In our setting where manifolds are contained within the hypercube $[0,1]^n$, restrictions on the reach can intuitively limit the space-filling capacity of a manifold resulting in upper bounds on the volume of the manifold.
Namely, as we will show below, for a sequence of manifolds $\{\mathcal{M}_n\}$ of intrinsic dimension $d$, if the reach is bounded by $\Reach{\mathcal{M}_n} = \omega(n^{0.5})$, then the volume of the manifold is $\Vol_d(\mathcal{M}_n) = \poly{(n)}$. This implies that the results proven in the previous subsections are tight up to $\Reach{\mathcal{M}_n} = \Theta(n^{0.5})$, where similar exponential hardness results may apply though are not immediately obtained from our proofs. We formally detail the tightness of our results in what follows below.

Our volume bounds will follow from Lemma B.1 of \cite{aamari2022adversarial} which locally bound the volume of a manifold of bounded reach contained within a ball.
\begin{lemma}[Adapted from Lemma B.1 of \cite{aamari2022adversarial}]\label{lem:aamari_bound_reach_volume}
    Consider a Riemannian manifold $\mathcal{M} \subset \mathbb{R}^n$ of intrinsic dimension $d$ with reach bounded by $R \leq \Reach{\mathcal{M}}$. Let $D$ be a measure on $\mathbb{R}^n$ supported on $\mathcal{M}$ with density function $f$ with respect to the volume measure $\Vol_d$ where $f$ is Lipschitz smooth and bounded by $0<f_{\min} \leq f(\vx) \leq f_{\max}$ for all $\vx \in \mathcal{M}$.  Denote $B(\vx, h) = \{\vx': \|\vx-\vx'\|\leq h\}$ as the ball centered around $\vx$ of radius $h$. Then for any $\vx_0 \in \mathbb{R}^n$ and $h \leq R/8$, it holds that 
    \begin{equation}
        D(B(\vx, h)) \leq (5/4)^{d/2} 2^d f_{\max} \omega_d \max \left\{\left( h^2- d(\vx, \mathcal{M})^2 \right)^{d/2}, 0 \right\}.
    \end{equation}
\end{lemma}

\begin{lemma} \label{lem:volume_upper_bound_reach}
    Denote $B(\vx, h) = \{\vx': \|\vx-\vx'\|\leq h\}$ as the ball centered around $\vx$ of radius $h$. Given a manifold $\mathcal{M} \subset \mathbb{R}^n$ of intrinsic dimension $d$ bounded in reach by $\Reach{\mathcal{M}} \geq R$, the volume of the intersection of the manifold and a ball $B(\vx, h)$ with $h\leq R/8$ is bounded by
    \begin{equation}
        \Vol_d\left( \mathcal{M} \cap B(\vx, h) \right) \leq \omega_d \left( \frac{R}{2} \right)^{d},
    \end{equation}
    where $\omega_d$ is the volume of the $d$-dimensional unit ball.
\end{lemma}
\begin{proof}
    This is a consequence of Lemma B.1 of \cite{aamari2022adversarial} recalled in \Cref{lem:aamari_bound_reach_volume} setting $f_{\min}=f_{\max}=1$ to recover the standard volume measure, i.e. for a given point $\vx \in \mathbb{R}^n$ and $h\leq \Reach{\mathcal{M}}/8$, \Cref{lem:aamari_bound_reach_volume} states
    \begin{equation}
        \Vol_d\left( \mathcal{M} \cap B(\vx, h) \right) \leq (5/4)^{d/2} 2^d \omega_d \left( h^2- d(\vx, \mathcal{M})^2 \right)^{d/2}.
    \end{equation}
    Applying the bound $h\leq \Reach{\mathcal{M}}/8$ and noting that $(5/4)^{d/2} 2^d (1/8)^d \leq (1/2)^{d} $ completes the proof.
\end{proof}

The above lemma provides an immediate bound on the volume as one can note that for large enough $n$, any manifold $\mathcal{M} \subset [0,1]^n$ is contained within a single ball $B(x, h)$ for $h = \omega(n^{0.5})$. 
\begin{proposition} \label{prop:tightness_of_results_for_reach}
    Given a sequence of manifolds $\{\mathcal{M}_n\}$ of intrinsic dimension $d$ (fixed and independent of $n$) and reach bounded by $\Reach{\mathcal{M}_n} = \omega(n^{0.5})$, the volume of the manifolds grows at most $\Vol_d(\mathcal{M}_n) = O(\poly(n))$.
\end{proposition}
\begin{proof}
    There exists large enough $n$ such that $[0,1]^n$ is contained within a ball of radius equal to the bound on the reach $\Reach{\mathcal{M}_n} = \omega(n^{0.5})$. Applying \Cref{lem:volume_upper_bound_reach} with $h = \sqrt{n}/2$ and $\vx = \bm 1 /2$ then implies that for large enough $n$:
    \begin{equation}
        \Vol_d\left( \mathcal{M}_n  \right) = \Vol_d\left( \mathcal{M}_n \cap B(\vx, h) \right) \leq \omega_d \left( \frac{\sqrt{n}}{4} \right)^{d} = \poly(n).
    \end{equation}
\end{proof}

\begin{remark}
    The above implies that manifolds with reach bounded as $\Reach{\mathcal{M}_n} = \omega(n^{0.5})$ fall within the class of efficiently sampleable manifolds defined in \Cref{def:efficiently_sampleable}. The proof above does not extend to the setting where $\Reach{\mathcal{M}_n} = \Omega(n^{0.5})$. Covering number bounds (e.g. see Theorem 2 of \cite{kossaczka2020entropy}) show that exponentially many (in $n$) balls of radius $cn^{0.5}$ are needed to cover the space $[0,1]^n$ when $c$ is sufficiently small. Therefore, \Cref{prop:tightness_of_results_for_reach} does not guarantee that volume is polynomially bounded in $n$ specifically for the setting where $\Reach{\mathcal{M}_n} = \Theta(n^{0.5})$.
\end{remark}


\section{Experiments confirming learnability results}
\label{app:hardness_experiment_details}
\begin{figure}[ht]
    \centering
\includegraphics{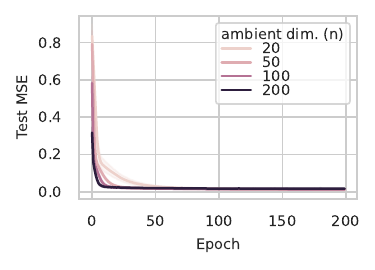}
    \caption{Replication of \Cref{fig:isoperimetric} where the target is a neural network of the same form but with randomly chosen weights. Results are aggregate over five random initializations.}
    \label{fig:experiment_random_NN}
\end{figure}

\begin{figure}[ht]
    \centering
\includegraphics{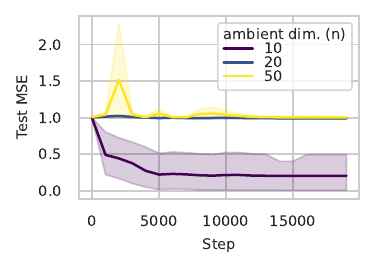}
    \caption{Replication of \Cref{fig:parity} where the trained neural network is overparameterized in both width ($3$ hidden layers instead of just $1$) and width. Results are consistent with that of \Cref{fig:parity} further corroborating the hardness of this learning task. Results are aggregated over five random initializations including random tuning of the optimization hyperparameters. }
    \label{fig:parity_3layers}
\end{figure}

For \Cref{fig:isoperimetric}, the data is generated by drawing random unit norm vectors from a subspace of dimension $d=10$ with ambient space of dimension $n$. The random $d$-dimensional subspace is obtained by drawing a random $n \times n$ orthogonal matrix and taking its first $d=10$ columns. Target functions for \Cref{fig:isoperimetric} are drawn from the class in Equation (5) in \cite{diakonikolas2020algorithms}, which form a hard class of single hidden layer neural networks to learn under the i.i.d. Gaussian input model. We set the hidden layer width $k=n/4 = O(n)$ for this construction. To corroborate these findings, we also show in \Cref{fig:experiment_random_NN} that the results are consistent when the objective is to learn a target function which is a single hidden layer ReLU neural network with randomly chosen weights of the same form. 
To normalize all target functions to have a consistent MSE benchmark, we divide by the norm of the function which is approximated by taking the average norm over a batch of $100$ randomly chosen inputs. 

We use the Pytorch package for our neural network experiments \cite{NEURIPS2019_9015}. Training a single neural network on the tasks shown in \Cref{fig:experiment_performance} took no longer than a few minutes on a single GPU. For the learnable setting in \Cref{fig:isoperimetric}, we use a training set of size $1000$ and  a neural network of a single hidden layer and width $100$ to learn the target function. For the hard setting of \Cref{fig:parity}, we provide the algorithm with a fresh randomly drawn batch of data each training step. The trained neural network is overparameterized with width $2n$ in these experiments. For the hard setting in \Cref{fig:parity}, we also attempted to learn the target function with an overparameterized network with three hidden layers and width $2n$. The results, shown in \Cref{fig:parity_3layers}, are consistent with those shown in \Cref{fig:parity}. In all experiments, the learning rate for the optimizer was varied by the default value times a random multiplicative factor $\exp(c)$ where $c \sim \operatorname{Unif}([-2,1])$; results were largely consistent across all ranges of the tuned values. 

\section{Estimating Intrinsic Dimension}
\label{app:experiments}

\subsection{Details on Experimental Setup}
\label{app:in-dim-exp}
Following \cite{stanczuk2022your}, to calculate the intrinsic dimension at a given point $\vp \in \mathcal{M}$, we sample a local neighborhood of points around $\vp$ by noising $\vp$ a small amount. Then, from a trained diffusion model, we collect a series of score vectors $\{\vs_i\}_{i=1}^N$ at the sampled neighborhood which point in the direction of the de-noised sample. We then apply PCA upon this collection of $N$ score vectors (treated as a matrix $\mS \in \mathbb{R}^{n \times N}$) which gives us a basis for the tangent space $T_{\vp}\mathcal{M}$ and normal space $N_{\vp}\mathcal{M}$. Given potential instabilities inherent in training diffusion models and performing statistical estimation on manifolds, we use the stable rank to estimate the rank of $\mS$ \cite{vershynin2010introduction}. Given a matrix $\mM$, the stable rank is given by
\begin{equation}
    r_S(\mM) \coloneqq \frac{\|\mM\|_F^2}{\|\mM\|_{\infty}^2},
\end{equation}
where $\|\cdot\|_F$ and $\|\cdot\|_{\infty}$ denote the Frobenius norm and maximum singular value respectively. Since $\mS$ is actually high rank, we estimate instead the rank of $s_{\max}\mI-\mS$ where $s_{\max}$ is the max singular value of $\mM$ to convert this into a low rank estimation problem. We find the stable rank above to be a more robust measure of estimated rank in comparison to other such metrics such as the maximum difference in singular values.

We note that this approach can also be taken by estimating the rank over a neighborhood of sampled points as well. That is, we find that collecting difference vectors between $\vp$ and the denoised points as samples also estimates the dimension of the tangent space $d$ accurately. Such a method may be more suited to higher dimensional inputs. However, to be consistent with prior work, we perform our analysis on collections of score vectors here. \cite{stanczuk2022your} have confirmed this overall methodology on synthetic datasets where the intrinsic dimension is known, showing that in general, it outperforms other statistical means to estimate intrinsic dimension.

\paragraph{Diffusion model and training.}

For Euclidean datasets, we train a fully-connected 5-layer ReLU network of width 2048 on 100000 samples for 30 epochs, batch size 64, and a learning rate of $4\times 10^{-4}$.

For image datasets, we pre-process the images to have pixel values in $[-1,1]$. We train a DDPM U-Net 2D model \cite{ronneberger2015u} on 10000 samples for 30 epochs, batch size 64, and a learning rate of $4\times 10^{-4}$. The architecture is held the same across all three datasets: a  ResNet block, a ResNet downsampling block with spatial self-attention, and a ResNet block with 64, 128, and 256 channels each (and the reverse for upsampling). 

We use a noise scheduler in the DDPM paradigm with 1000 timesteps between the image and pure Gaussian noise, with the squared cosine beta schedule introduced by \cite{nichol2021improved}. This training was done on an NVIDIA L4 24GB GPU.

We investigate synthetic unit hyperspheres to sanity check our method and proceed to three real-world image datasets: MNIST~\cite{mnist}, Fashion MNIST (FMNIST)~\cite{fmnist}, and Kuzushiji-MNIST (KMNIST)~\cite{kmnist}. For each of the three image datasets, we consider two sizes: the original $28 \times 28$ and a downsampled $12 \times 12$.

\subsection{Results on synthetic data manifolds}
\label{app:synth-mf}

Following \cite{stanczuk2022your}, we verify our intrinsic dimension estimation procedure on unit hyperspheres of different dimension, randomly projected into a higher-dimensional ambient space. We vary intrinsic and ambient dimension and report estimates of intrinsic dimension measurements (Table \ref{tab:sphere_dim}). We observe that, generally, the estimated and true intrinsic dimension align well. Our results further suggest an overall higher accuracy if the intrinsic dimension is small compared to the ambient dimension. Figure \ref{fig:sphere_dim} shows the singular values of the score vector for different configurations; the curves, resembling a step function, distinguish normal from tangent directions. 

\begin{table}[ht]
\centering
\begin{tabular}{ccc}
\toprule
\textbf{Ambient Dim.} & \textbf{Intrinsic Dim.} & \textbf{Estimated Intrinsic Dim.} \\
\midrule
20 & 2 & 3 \\
20 & 10 & 10 \\
20 & 18 & 17 \\
\midrule
100 & 10 & 12 \\
100 & 50 & 47 \\
100 & 90 & 82 \\
\bottomrule
\end{tabular}
\caption{Estimated intrinsic dimension for the hypersphere.}
\label{tab:sphere_dim}
\end{table}

\begin{figure}[h]
    \centering
    \begin{subfigure}[b]{0.33\textwidth}
        \centering
        \includegraphics[width=\textwidth]{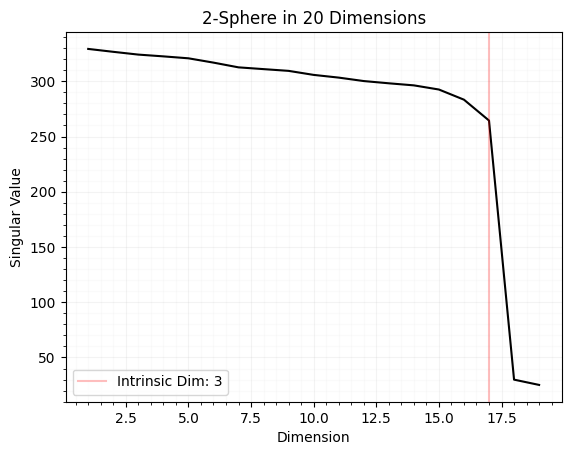}
    \end{subfigure}~
    \begin{subfigure}[b]{0.33\textwidth}
        \centering
        \includegraphics[width=\textwidth]{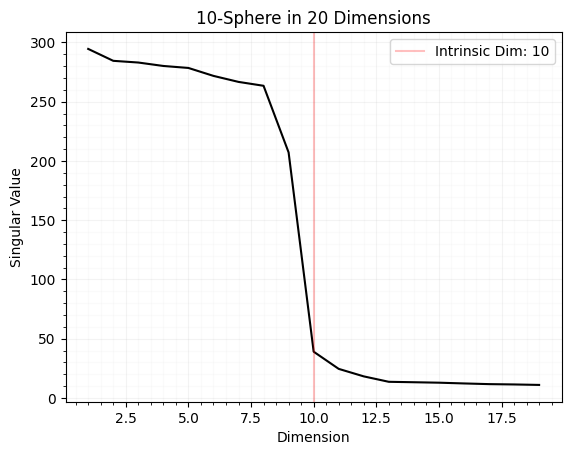}
    \end{subfigure}~
    \begin{subfigure}[b]{0.33\textwidth}
        \centering
        \includegraphics[width=\textwidth]{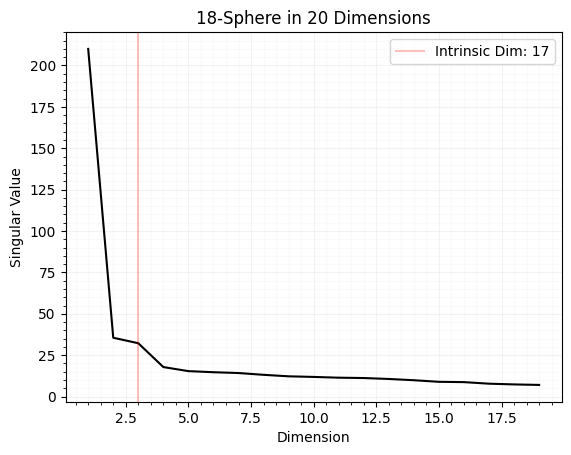}
    \end{subfigure}
        \begin{subfigure}[b]{0.33\textwidth}
        \centering
        \includegraphics[width=\textwidth]{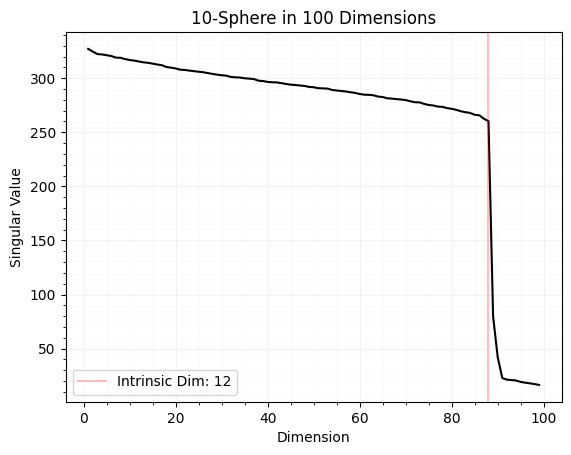}
    \end{subfigure}~
    \begin{subfigure}[b]{0.33\textwidth}
        \centering
        \includegraphics[width=\textwidth]{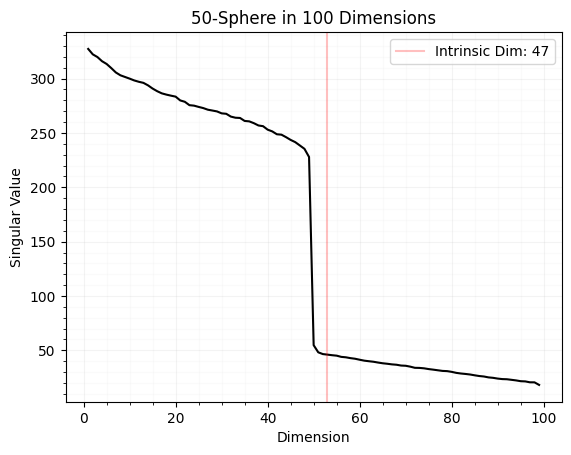}
    \end{subfigure}~
    \begin{subfigure}[b]{0.33\textwidth}
        \centering
        \includegraphics[width=\textwidth]{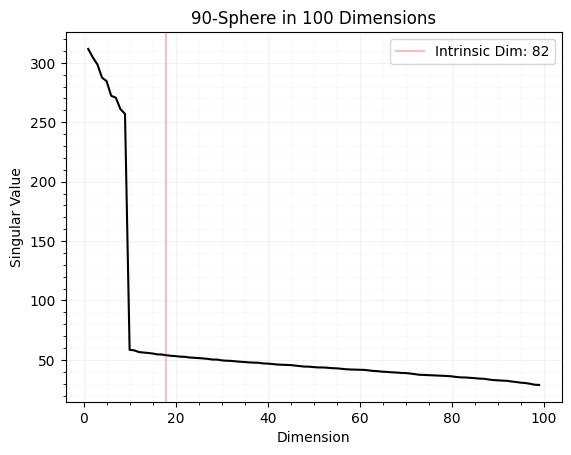}
    \end{subfigure}
    \caption{Intrinsic Dimension Estimation on Hyperspheres}
    \label{fig:sphere_dim}
\end{figure}

\subsection{Results on real-world data manifolds}
\label{app:real-mf}

\begin{table}[ht]
\centering
\begin{tabular}{lcc}
\toprule
\textbf{Dataset} & \textbf{Intrinsic Dimension (Std Dev)} \\
\midrule
MNIST (12 x 12) & 29.8 (7.0) \\
MNIST (28 x 28) & 97.0 (18.4)\\
KMNIST (12 x 12) & 52.9 (9.4) \\
KMNIST (28 x 28) & 168.9 (28.6) \\
FMNIST (12 x 12) & 39.8 (10.5) \\
FMNIST (28 x 28) & 299.1 (135.9) \\
\bottomrule
\end{tabular}
\caption{Estimated intrinsic dimension for the data manifolds considered. We report the mean across classes with standard deviation in brackets.}
\label{tab:dim}
\end{table}

As stated in the main text, we observe that our estimate of MNIST's intrinsic dimension matches that obtained in \cite{stanczuk2022your}. We observe slightly lower values, which we believe to be due to different input sizes. In~\cite{stanczuk2022your}, MNIST images were upscaled to size $32 \times 32$, which differs from our input sizes ($28 \times 28$, $12 \times 12$). 
In addition to MNIST, we also analyze FMNIST and KMNIST, which have not been considered in previous studies.
Our results suggest that among the three datasets, FMNIST has the largest intrinsic dimension. Images in FMNIST cover part of the image in comparison to MNIST and KMNIST; we believe that much of this increase in intrinsic dimension is due to the larger set of possible perturbations of the image within this area. 

We also explore how resizing the image impacts the data manifold's intrinsic dimension in \Cref{fig:resize} and \Cref{tab:dim}. We see that for MNIST and KMNIST, resizing to a smaller ambient dimension can reduce the intrinsic dimension though at a slightly slower rate than that of the ambient dimension. We note that all the images we consider here have relatively intrinsic dimension, where the effects of resizing may be more prominent.

\begin{figure*}[ht]
    \centering
    \includegraphics[width=0.33\linewidth]{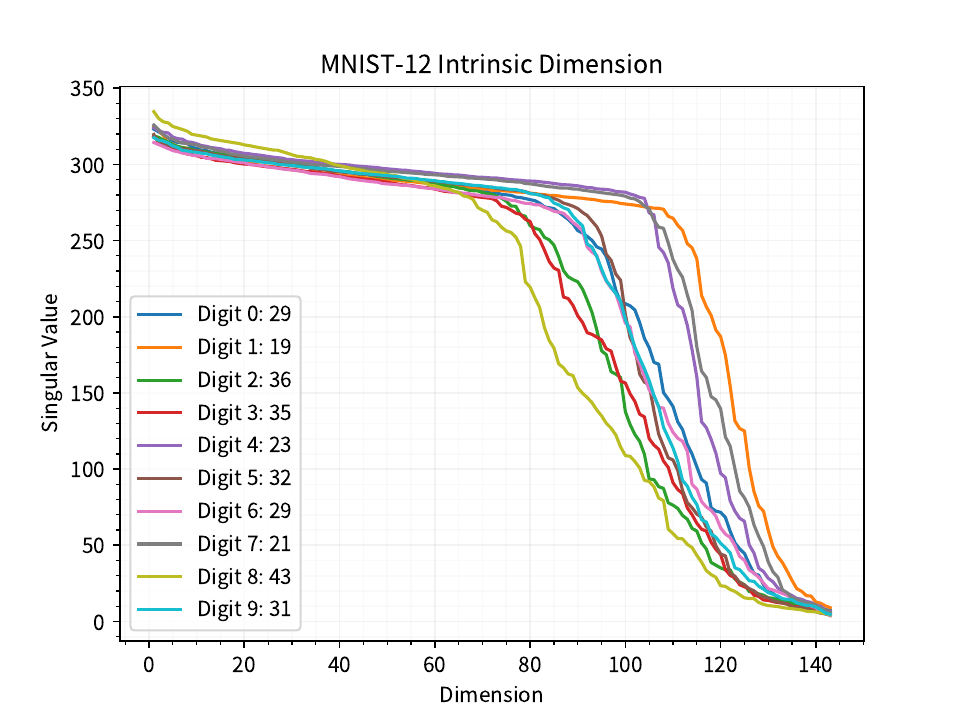}~
    \includegraphics[width=0.33\linewidth]{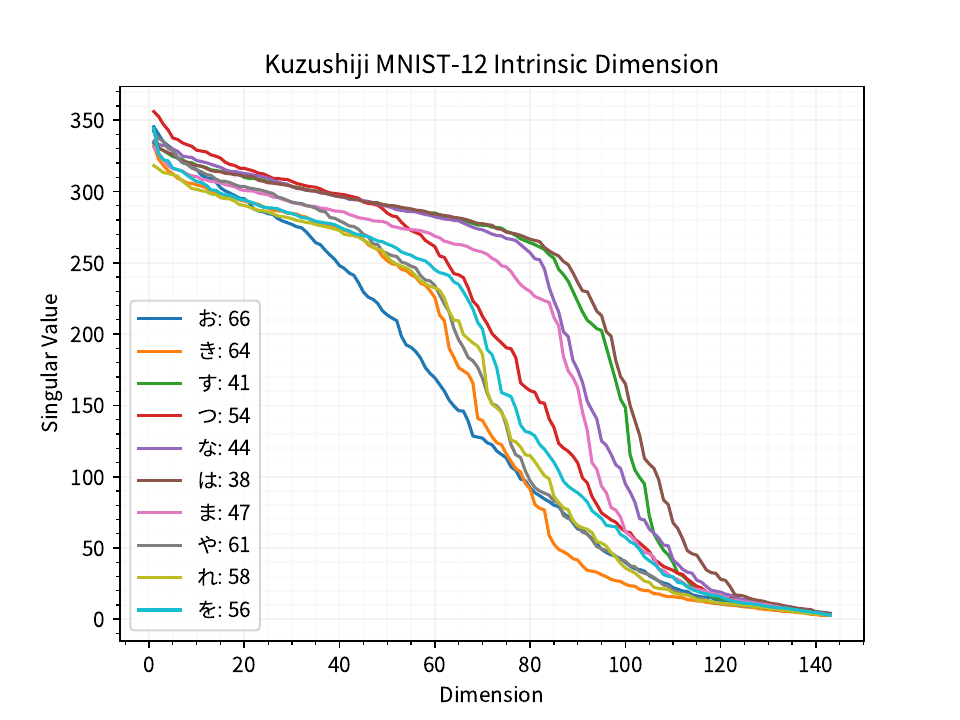}~
    \includegraphics[width=0.33\linewidth]{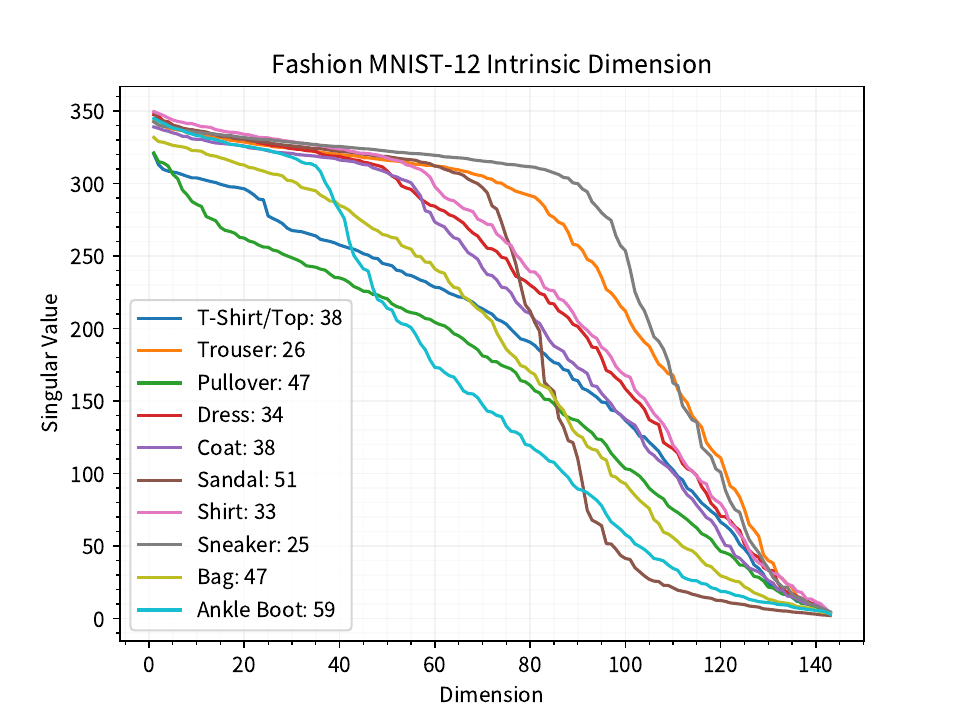}
    \caption{Impact of Resizing on Score Spectrum for Estimating Intrinsic Dimension from Diffusion Models}
    \label{fig:resize}
\end{figure*}

\end{document}